\documentclass{article}

\usepackage[preprint,nonatbib]{neurips_2026}
\usepackage{amsmath,amssymb,amsthm,mathtools,bm}
\usepackage{booktabs,array}
\usepackage{graphicx}
\usepackage{enumitem}
\usepackage{microtype}
\usepackage{placeins,needspace}
\usepackage{xcolor}
\usepackage[authoryear]{natbib}
\usepackage{xurl}
\usepackage{caption}
\usepackage{etoolbox}
\AtBeginEnvironment{table}{\small}
\AtBeginEnvironment{table*}{\small}
\usepackage{hyperref}
\definecolor{paperlink}{RGB}{35,65,105}
\hypersetup{colorlinks=true,linkcolor=paperlink,citecolor=paperlink,urlcolor=paperlink,
  pdftitle={Two-Stage Score Estimation for Group-Specific Diffusion Models},
  pdfauthor={Huitao Yang; Oscar Hernan Madrid Padilla}}
\newtheorem{theorem}{Theorem}
\newtheorem{proposition}{Proposition}
\newtheorem{lemma}{Lemma}
\newtheorem{corollary}{Corollary}
\newtheorem{assumption}{Assumption}
\newtheorem{definition}{Definition}

\newtheorem{example}{Example}

\newcommand{\R}{\mathbb{R}}

\newcommand{\E}{\mathbb{E}}
\newcommand{\Pp}{\mathbb{P}}
\newcommand{\one}{\mathbf{1}}
\newcommand{\calF}{\mathcal{F}}
\newcommand{\calG}{\mathcal{G}}

\newcommand{\calN}{\mathcal{N}}
\newcommand{\calR}{\mathcal{R}}
\newcommand{\calD}{\mathcal{D}}
\newcommand{\Tmin}{\underline T}
\newcommand{\Tmax}{\overline T}
\newcommand{\dd}{\mathrm{d}}
\newcommand{\polylog}{\operatorname{polylog}}
\newcommand{\norm}[1]{\left\lVert #1\right\rVert}

\DeclareMathOperator{\TV}{TV}
\DeclareMathOperator{\KL}{KL}
\DeclareMathOperator{\Var}{Var}
\DeclareMathOperator{\Cov}{Cov}

\title{Two-Stage Score Estimation for Group-Specific Diffusion Models}
\author{%
  Huitao Yang\textsuperscript{1}\quad
  Oscar Hernan Madrid Padilla\textsuperscript{1}\\[3pt]
  {\small\textsuperscript{1}University of California, Los Angeles\quad}
}
\date{}
\begin{document}
\maketitle
\raggedbottom

\begin{abstract}
Generative modeling across related distributions raises a question: how can shared structure support accurate generation for a target distribution with limited observations?
We develop a statistical framework for diffusion transfer that analyzes a shared pooled score together with a target-specific offset.
Under Besov regularity of the pooled density and a stable hierarchical composition model for the offset, we derive convergence rates that separate the costs of learning these components and identify when a simpler offset yields a faster risk-bound rate than separate fitting.
Under suitable boundary and regularity assumptions, we prove that the offset remains bounded and jointly regular as diffusion time approaches zero.
These results allow the lower time cutoff to vanish polynomially and provide total-variation guarantees for diffusion models as density estimators.
Synthetic simulations and image benchmarks provide a proof of concept for the framework.
\end{abstract}

\section{Introduction}
\label{sec:introduction}

Diffusion models learn a data distribution through forward noising and a
learned reverse process~\citep{intro_sohldickstein2015,ho2020}. In
score-based formulations, generation relies on log-density gradients
across noise levels~\citep{intro_songermon2019,song2021}, estimated using
score matching and its denoising formulation
\citep{hyvarinen2005,vincent2011}. Recent theory establishes near-minimax
distribution-estimation rates under Besov regularity and additional
support conditions~\citep{oko2023}, and adaptation to low-dimensional
linear structure~\citep{chenhuang2023}.
In this setting, we study diffusion models for related groups in
$\mathbb R^d$, each with its own data distribution.
Separate training uses only the target group's limited data, whereas
pooled training uses all groups but learns their combined distribution.
This motivates sharing data while preserving group differences, a central
goal of multitask and transfer learning
\citep{intro_caruana1997,du2017,caipu2024}, including diffusion models
with limited target data~\citep{ouyang2024}.

Existing diffusion transfer theory exploits shared latent structure
\citep{yang2024}, shared conditioning representations~\citep{cheng2025},
and density-ratio guidance~\citep{ouyang2024}. Related work learns score
offsets through residual denoising~\citep{intro_verine2025} or
Sobolev-regularized density-ratio estimation
\citep{xie2026scoredifference}. We analyze a two-stage estimator in which
the shared score and group-specific offset may have different
structural complexity.
We denote the pooled and group-$\ell$ diffused
densities by $p_t$ and $p_{t,\ell}$, with scores $s_t=\nabla\log p_t$ and
$s_{t,\ell}=\nabla\log p_{t,\ell}$. The density-ratio
identity~\citep{ouyang2024}
\(s_{t,\ell}=s_t+\Delta_{t,\ell}, \quad
\Delta_{t,\ell}=\nabla\log(p_{t,\ell}/p_t)\)
suggests first estimating $s_t$ from $n$ pooled observations, then
freezing this fit and learning an offset from an independent sample
of $n_\ell$ target observations, paralleling the pooled-plus-offset
construction of \citet{ren2026}. We model the offset using a hierarchical
composition class in $(x,t)$~\citep{schmidthieber2020,kohlerlanger2020},
whose complexity depends on constituent smoothness and input arity.

We first separate pooled-score error from offset approximation and
estimation costs in Theorem~\ref{thm:oracle}, then combine Besov and
hierarchical regularity to derive structural rates and positive-transfer
conditions in Theorem~\ref{thm:mixed-main}. We establish density-level
conditions for small-time offset stability in
Propositions~\ref{prop:ratio-cancellation} and~\ref{prop:offset}, and
translate score estimation into total-variation accuracy in
Corollary~\ref{cor:tv-rate}.
Our main contribution is a theoretical framework for two-stage transfer
learning in score-based diffusion models, linking statistical transfer
guarantees to the boundedness and regularity of target-specific score
offsets in the vanishing-noise regime.
Specifically:
\begin{itemize}[leftmargin=1.5em]
\item \textbf{A framework for two-stage diffusion transfer.}
We develop a statistical framework for learning a target score through a
pooled fit and a target-specific offset. The conditional oracle inequality
in Theorem~\ref{thm:oracle} separates inherited pooled-score error from
offset approximation and estimation costs, so that target-stage complexity
depends only on the trainable offset class.

\item \textbf{Structural rates and transfer gains.}
Under Besov regularity of the pooled density and a stable hierarchical
composition model for the offset, Theorem~\ref{thm:mixed-main} bounds the
integrated squared score risk by
\(\widetilde O\!\left(
n^{-2s_0/(2s_0+d)}+n_\ell^{-\alpha_\Delta}
\right),\)
where $s_0$ is pooled-density smoothness and $\alpha_\Delta$ is determined
by offset-constituent smoothness and arity. This separation quantifies
when pooled data and a structurally simpler offset yield a faster
upper-bound rate than separate fitting.

\item \textbf{Offset stability and density estimation at vanishing noise.}
We establish density and boundary conditions under which the offset
remains bounded and jointly regular in space and time as $t\to0$, even
when the individual scores become singular. This analysis supports
polynomially vanishing lower time cutoffs and, through
Corollary~\ref{cor:tv-rate}, total-variation guarantees for diffusion
models as estimators of the target data distribution.
\end{itemize}

\section{Related work}
\paragraph{Nonparametric transfer learning.}
Nonparametric transfer learning exploits structural relationships between
source and target tasks. Earlier analyses study smooth model shifts and
offset estimation~\citep{revision_wangschneider2015,revision_wangetal2016},
transformation-based transfer~\citep{du2017}, and adaptation to differing,
unknown regularity~\citep{caipu2024,revision_linreimherr2024}. Minimax and
adaptive analyses also cover nonparametric classification
\citep{related_caiwei2021,related_reeve2021} and covariate shift in
regression~\citep{related_ma2023}. \citet{related_jiao2024} analyze deep
transfer through shared and domain-specific features, whereas
\citet{ren2026} provide the closest grouped-regression analogue: a pooled
estimator plus group-specific offsets with hierarchical ReLU rates.
We study the corresponding separation for diffusion scores, accounting
for vanishing noise and distributional recovery.

\paragraph{Transfer and guidance in diffusion models.}
Diffusion transfer uses density-ratio guidance~\citep{ouyang2024}, shared
latent linear structure~\citep{yang2024}, or shared conditioning
representations~\citep{cheng2025}. We include the classifier-based method
of \citet{ouyang2024} as a baseline. Related conditional-density theory
exploits covariate smoothness and manifold structure~\citep{tang2025}.
Our structural assumption concerns the group-to-pooled score offset,
allowing a complex pooled distribution to be shared across groups whose
offsets have lower constituent arity or greater smoothness.
Classifier guidance adds classifier gradients~\citep{dhariwal2021}, while
classifier-free guidance combines conditional and unconditional
scores~\citep{hosalimans2022}. Discriminator guidance refines pretrained
scores through density-ratio gradients~\citep{related_kim2023}, and
\citet{intro_verine2025} train gradient offsets using residual
denoising objectives. \citet{xie2026scoredifference} establish convergence
guarantees for Sobolev-regularized score-difference estimation, including
time-dependent diffusion. We share the density-ratio identity and
additive adaptation with these approaches, while tracking
pooled-fit error and the offset's compositional complexity in a joint
bound in $n$ and $n_\ell$, quantifying when a simpler offset improves
the bound over separate fitting.

\paragraph{Statistical theory for diffusion models.}
Score matching~\citep{hyvarinen2005}, its denoising
identity~\citep{vincent2011}, and reverse-SDE generation~\citep{song2021}
underlie our procedure. \citet{oko2023} establish nearly minimax
distribution-estimation rates; their joint space-time Besov approximation
supplies our pooled-stage rate. \citet{chenhuang2023} instead exploit
low-dimensional linear data geometry. Hierarchical ReLU approximation
supplies constituent-based offset rates
\citep{schmidthieber2020,kohlerlanger2020}. Small-time cancellation
supports regularity, but the compositional representation remains an
assumption. Score-to-distribution analyses control generation error from
approximate scores and numerical discretization
\citep{related_chen2023sampling,chenlee2023}.

\section{Methodology}
\label{sec:methodology}
In this section, we present the two-stage transfer learning method analyzed in our theoretical framework.
The unconditional pooled score function describes the shared structure of separate group distributions (Section~\ref{sec: setting}). We first
estimate this score from pooled observations without group labels (Section~\ref{sec:offset-estimation}), then freeze the shared backbone and 
learn a target-group offset from group-specific samples.

\subsection{Grouped data and diffusion scores}
\label{sec: setting}
We denote the density of group $\ell\in[L]$ by $p_{0,\ell}$ and its
pooled prevalence by $\pi_\ell>0$, with $\sum_{\ell=1}^L\pi_\ell=1$. The pooled density is
$p_0=\sum_{\ell=1}^L\pi_\ell p_{0,\ell}$.
In Stage~1, we use an i.i.d. pooled training sample
$(X_i,Z_i)_{i=1}^n$ from different groups, where $X_i\in\R^d$, $\Pp(Z_i=\ell)=\pi_\ell$, and
$X_i\mid Z_i=\ell$ has density $p_{0,\ell}$.
Stage~2 uses an independent i.i.d. target-group sample
$(X_{\ell,i})_{i=1}^{n_\ell}$ from $p_{0,\ell}$, with $n_\ell\leq n$.
These sample sizes describe the two training stages directly. A single
grouped dataset can supply the independent samples by sample splitting;
constant-fraction splitting changes the rates only by constants.

We consider the Ornstein--Uhlenbeck (OU) process
\[
    \dd X_t=-\beta_tX_t\dd t+\sqrt{2\beta_t}\dd B_t,
    \qquad
    0<\underline\beta\leq\beta_t\leq\overline\beta,
\]
whose transition density $q_t(\cdot\mid x_0)$ is
$N(m_tx_0,\sigma_t^2I_d)$, where
\[
    m_t=\exp\!\left(-\int_0^t\beta_u\dd u\right),
    \qquad
    \sigma_t^2=1-m_t^2.
\]
For a learning interval $0<\Tmin<\Tmax$, write
$\sigma_-:=\sigma_{\Tmin}$ and $\sigma_+:=\sigma_{\Tmax}$.
Our setting is a variance-preserving diffusion in the score-SDE
framework of \citet{song2021}, with the stated normalization
of $\beta_t$. We denote the resulting pooled and group-$\ell$ densities
by $p_t$ and $p_{t,\ell}$, and write
\(s_t=\nabla\log p_t,
    s_{t,\ell}=\nabla\log p_{t,\ell}.\)

\subsection{Two-stage offset estimation}
\label{sec:offset-estimation}
For a vector field $f:\R^d\times[\Tmin,\Tmax]\to\R^d$, we use the integrated
group score risk
\begin{equation}
\calR^{\rm int}_\ell(f)
:=
\int_{\Tmin}^{\Tmax}\int_{\R^d}
\norm{f(x,t)-s_{t,\ell}(x)}_2^2p_{t,\ell}(x)\dd x\dd t,
\label{eq:risk}
\end{equation}
where $\|\cdot\|_2$ is the Euclidean norm and the superscript ${\rm int}$ denotes time integration.
The denoising-score-matching loss replaces $s_{t,\ell}(X_t)$ with the known
transition score
\begin{equation}
\label{eq:transition-score}
\xi_t(x,x_0):=\nabla_x\log q_t(x\mid x_0)
    =-\frac{x-m_tx_0}{\sigma_t^2}.
\end{equation}
For the target-group sample, its empirical loss is
\[
\begin{aligned}
\calD^{\rm int}_{\ell,n_\ell}(f)
&:=\frac1{n_\ell}\sum_{i=1}^{n_\ell}\calD^{\rm int}_{\ell,i}(f),\\
\calD^{\rm int}_{\ell,i}(f)
&:=\int_{\Tmin}^{\Tmax}\E_{X_t\mid X_{\ell,i}}
\bigl[\norm{f(X_t,t)-\xi_t(X_t,X_{\ell,i})}_2^2\bigr]\dd t.
\end{aligned}
\]
Here, $\E_{X_t\mid X_{\ell,i}}$ averages over the forward-corruption
law $X_t=m_tX_{\ell,i}+\sigma_t\varepsilon$, with
$\varepsilon\sim N(0,I_d)$, conditional on the observed clean sample
$X_{\ell,i}$.
Likewise, the pooled score estimation uses the same loss expression as above but with unconditional $X_i$ and averages over
$i=1,\ldots,n$, with corresponding notations $\calD_n^{\rm int}, \calD_i^{\rm int}$. Vincent's identity implies that the population denoising
loss and \eqref{eq:risk} differ by a constant independent of $f$
\citep{vincent2011}. Minimizing the denoising loss therefore targets the
integrated score risk.

The population offset for group $\ell$ is denoted $\Delta_{t,\ell}=s_{t,\ell}-s_t$. We denote $s(x,t)=s_t(x)$, $s_\ell(x,t)=s_{t,\ell}(x)$, and
$\Delta_\ell(x,t)=\Delta_{t,\ell}(x)$ for the joint space--time fields.
For $n\geq2$ and a fixed Gaussian-tail exponent $a>0$, we set
\(C_n=\sqrt{2(a+2)\log n},\quad
K_n(t)=\frac{\sqrt d\,C_n}{\sigma_t}.\)
We write $\Pi_t:\R^d\to\R^d$ for Euclidean projection onto the closed ball
\(B_2(0,K_n(t))=\{v\in\R^d:\|v\|_2\leq K_n(t)\}.\)
Explicitly, $\Pi_t(v)=v\min\{1,K_n(t)/\|v\|_2\}$ for $v\ne0$,
and $\Pi_t(0)=0$.
For a vector-valued function
$f:\R^d\times[\Tmin,\Tmax]\to\R^d$, we write
$(\Pi f)(x,t)=\Pi_t(f(x,t))$.

In Stage~1, we estimate the pooled score by
$\widehat s\in\arg\min_{f\in\mathcal F_n}\mathcal D_n^{\rm int}(f)$.
Then in Stage~2, we keep this fit fixed and estimate the offset by
\begin{equation}
    \widehat g_\ell\in\arg\min_{g\in\calG_\ell}
    \calD^{\rm int}_{\ell,n_\ell}\!\left(\Pi(\widehat s+g)\right),
    \qquad
    \widehat s_\ell=\Pi(\widehat s+\widehat g_\ell).
\label{eq:estimator}
\end{equation}
For the convergence results, $\mathcal F_n$ and $\mathcal G_\ell$ use
sparse and dense ReLU networks, respectively, with bounded coefficients, borrowing techniques from~\cite{kohlerlanger2020}.
Appendices~\ref{app:stable-construction}--\ref{app:mixed-main} specify
their sizes and clipping rules. The theory evaluates corruption
expectations and time integrals exactly and analyzes Stage~2 conditional
on the pooled fit.

\section{Theory}
\label{sec:theory}
We first separate pooled and offset errors in Theorem~\ref{thm:oracle},
then apply Besov and compositional approximation to derive the rates in
Theorem~\ref{thm:mixed-main}. We then study small-time regularity to
support offset stability and use Corollary~\ref{cor:tv-rate} to obtain
distributional accuracy from score accuracy.

For a vector-valued
function on the learning domain, write
\[
\|f\|_{2,\infty}
:=\sup_{(x,t)\in\R^d\times[\Tmin,\Tmax]}\|f(x,t)\|_2,
\]
and denote radius-$\eta$ covering numbers by $\calN(\eta,\calF,\|\cdot\|)$.
We write $a_n\lesssim b_n$ for $a_n\leq Cb_n$ with sample-independent
$C$, and $a_n\asymp b_n$ when both directions hold.
We use $\polylog(u)$ for a fixed polynomial in $\log u$ and
$\widetilde O$ to suppress these factors. Additional notation appears in Appendix~\ref{app:notation}.

\subsection{Error decomposition for two-stage estimation}
\label{sec:general-bound}
We first consider a general offset class satisfying the following entropy condition (Assumption~\ref{as:oracle})
and choose a learning interval obeying
\begin{equation}
    n^{-\kappa_0}\leq\Tmin\leq t_0,
    \qquad
    \Tmax\asymp\log n,
    \qquad
    \Tmax-\Tmin\geq c_0>0,
\label{eq:window}
\end{equation}
where $\kappa_0,c_0,t_0$ are fixed positive constants.
Condition~\eqref{eq:window} allows the lower cutoff to decrease
polynomially with $n$, while the upper cutoff grows logarithmically.

\begin{assumption}[Offset entropy and available Stage~1 bound]
\label{as:oracle}
Conditional on $\widehat s$, for $0<\eta\leq1$,
\[
    \log\calN\bigl(\eta,\calG_\ell,\norm{\cdot}_{2,\infty}\bigr)
    \leq V_{n,\ell}\log(A_{n,\ell}/\eta),
    \qquad A_{n,\ell}\geq e.
\]

For the oracle analysis, suppose that Stage~1 has already supplied a
deterministic error bound $\bar r_{n,\ell}(\delta_0)$ satisfying
\(\Pp\!\left(r^{\rm int}_{n,\ell}
\leq\bar r_{n,\ell}(\delta_0)\right)\geq1-\delta_0,\)
where
\[
r^{\rm int}_{n,\ell}:=
\int_{\Tmin}^{\Tmax}\int
\norm{\widehat s(x,t)-s_t(x)}_2^2p_{t,\ell}(x)\dd x\dd t.
\]
\end{assumption}
Here, the entropy bound quantifies the effective complexity of the offset class through its covering number. 
\(r^{\rm int}_{n,\ell}\) measures how accurately the Stage 1 fit estimates the pooled score on the noised target-group distribution, integrated over the learning interval. Under the common diffusion,
the overlap condition $p_{0,\ell}\leq\bar\kappa p_0$ implies
$p_{t,\ell}\leq\bar\kappa p_t$, transferring pooled risk to this weighting.
We carry the available Stage~1 bound $\bar r_{n,\ell}$ into the oracle
and derive it from Besov regularity and overlap in
Theorem~\ref{thm:mixed-main}.
\begin{theorem}[Error bound for two-stage score estimation]
\label{thm:oracle}
Suppose Assumption~\ref{as:oracle} and \eqref{eq:window} hold, and let
$0<\delta\leq1/2$. Conditional on $\widehat s$ and its first-stage risk
event, with probability at least $1-\delta$ over the independent
target-group sample,
\begin{equation}
\begin{aligned}
\calR^{\rm int}_\ell(\widehat s_\ell)
\leq C\Bigg[&\bar r_{n,\ell}+\mathsf A^{\rm int}_{n,\ell}
+\tau^{\rm int}_{n,\ell}+\frac{d\log^2n}{n_\ell}
\left\{V_{n,\ell}\log\!\bigl(A_{n,\ell}n_\ell K_n^{\max}\bigr)
+\log(1/\delta)\right\}\Bigg].
\end{aligned}
\label{eq:oracle}
\end{equation}
where $\bar r_{n,\ell}=\bar r_{n,\ell}(\delta_0)$,
$K_n^{\max}=\sqrt dC_n/\sigma_-$,
\[
\mathsf A^{\rm int}_{n,\ell}:=
\inf_{g\in\calG_\ell}\int_{\Tmin}^{\Tmax}\int
\norm{g(x,t)-\Delta_{t,\ell}(x)}_2^2p_{t,\ell}(x)\dd x\dd t,
\]
and $\tau^{\rm int}_{n,\ell}\leq\operatorname{poly}(d,\log n)n^{-(a+2)}$
is a Gaussian-tail remainder.
\end{theorem}
\noindent\emph{Proof idea.} The pooled-fit error comes from Stage~1; the other principal
terms measure the offset's approximation and estimation costs.
We separate these costs using the comparator identity
$\widehat s+g-s_\ell=(\widehat s-s)+(g-\Delta_\ell)$ and control
time-dependent clipping through its integrated scale:
\begin{equation}
\int_{\Tmin}^{\Tmax}\sigma_t^{-2}\dd t
\lesssim\log(1/\Tmin)+\Tmax\lesssim\log n.
\label{eq:main-time-integral}
\end{equation}
Thus even polynomial early stopping contributes only logarithmically
to the empirical-loss scale; the maximum clipping radius appears only
inside the entropy logarithm in \eqref{eq:oracle}.
Appendix~\ref{app:oracle} gives the truncation, concentration, and
projection arguments.

\subsection{Convergence rates and gains from pooling}
\label{sec:combined}
\label{sec:stable-offset}
\label{sec:pooled-besov}

To obtain explicit rates, we impose Besov regularity on the pooled density and a stable hierarchical composition structure on the offset. The former allows general density smoothness without prescribing a particular structure, while the latter captures the potentially simpler structure of the offset, whose small-time regularity is examined in Section~\ref{sec:score-properties}. 

We first introduce the definition of hierarchical composition, where functions are determined by constituents in a layer-by-layer style.
Following \citet{kohlerlanger2020}, each constituent is characterized by its smoothness \(q\) and arity \(K\), the number of its inputs (including time when used). We use the H\"older convention specified in
Appendix~\ref{app:notation}.

\begin{definition}[Coordinatewise hierarchical composition]
\label{def:general-hierarchy}
Let $\mathcal P\subset(0,\infty)\times\mathbb N$ be finite and nonempty.
The class $\mathcal H_D^{\rm scal}(0,\mathcal P,C)$ consists of
coordinate projections on $\R^D$. Recursively,
$h\in\mathcal H_D^{\rm scal}(M,\mathcal P,C)$ if
$h=g(h_1,\ldots,h_K)$ for some $(q,K)\in\mathcal P$, where
$g$ is globally bounded, $C$-Lipschitz, and has
$C^q$-norm at most $C$, and
$h_i\in\mathcal H_D^{\rm scal}(M_i,\mathcal P,C)$
with $M_i<M$.
For $\mathbf M=(M_j)_{j=1}^d$ and
$\mathbf P=(\mathcal P_j)_{j=1}^d$,
$\mathcal H_{D,d}^{\rm coord}(\mathbf M,\mathbf P;C)$
consists of vector functions whose $j$th coordinate is zero
or belongs to $\mathcal H_D^{\rm scal}(M_j,\mathcal P_j,C)$.
\end{definition}

We now assume that the true score offset admits a hierarchical composition structure with controlled complexity.

\begin{assumption}[Stable offset composition]
\label{as:stable-offset}
For each $n$ and learning interval satisfying \eqref{eq:window},
the offset admits, on
$Q_n=[-R_n,R_n]^d\times[\Tmin,\Tmax]$,
$R_n=1+C_R\sqrt{\log n}$, a representation
\[
\Delta_\ell=a^\Delta_{n,\ell}\widetilde\Delta_{n,\ell},
\qquad
\widetilde\Delta_{n,\ell}\in
\mathcal H_{d+1,d}^{\rm coord}
(\mathbf M_\ell,\mathbf P_\ell;C),
\]
where $1\leq a^\Delta_{n,\ell}\leq C(\log n)^c$.
The structural parameters $\mathbf M_\ell,\mathbf P_\ell,C$
are fixed independently of $n,n_\ell$ and the learning interval.
The constant $C_R$ is chosen as specified in
Appendix~\ref{app:stable-construction}.
\end{assumption}

Assumption~\ref{as:stable-offset} captures the structural simplicity of the offset, with controlled growth as the lower time cutoff decreases. We next impose a more general smoothness condition on the pooled density, without requiring a specific compositional structure.

\begin{assumption}[Pooled Besov model]
\label{as:pooled-besov}
The pooled density is supported on $[-1,1]^d$, is bounded above by
$C_0$ and below by $c_0'>0$ on that cube, and its restriction to the cube belongs to a fixed
$B^{s_0}_{p_0^\ast,\theta_0^\ast}([-1,1]^d)$ ball, where
$0<p_0^\ast,\theta_0^\ast\leq\infty$ and
$s_0>d(1/p_0^\ast-1/2)_+$.
Further, we assume the same boundary and
diffusion-schedule conditions as in \citet[Theorem~3.1]{oko2023};
Lemma~\ref{lem:proof-oko-input} states them explicitly.
The group density obeys
$p_{0,\ell}\leq\bar\kappa p_0$ for a fixed $\bar\kappa$ independent
of group prevalence.
\end{assumption}
Here $s_0$ measures the Besov smoothness of the pooled density,
while $(q,K)$ specifies the smoothness and arity of an offset
constituent. Under these assumptions,
Theorem~\ref{thm:mixed-main} separates the pooled and offset
contributions to the score-risk bound. The offset exponent
and corresponding rate are
\begin{equation}
\begin{aligned}
\alpha_\Delta&:=\min_{\substack{j\in[d]\\(q,K)\in\mathcal P_j^\ell}}
\frac{2q}{2q+K},\\
\rho_{n,\ell}&:=\rho_{n_\ell}(\mathbf P_\ell)=n_\ell^{-\alpha_\Delta},
\end{aligned}
\label{eq:offset-exponent}
\end{equation}
omitting identically zero coordinates. If the offset is identically
zero and we choose $\calG_\ell=\{0\}$, its rate term is omitted.

\begin{theorem}[Convergence rate for two-stage score estimation]
\label{thm:mixed-main}
Assume \eqref{eq:window}, Assumptions~\ref{as:stable-offset}
and \ref{as:pooled-besov}, independent first- and second-stage samples,
and $C\log n\leq n_\ell\leq n$ for a sufficiently large fixed
constant $C>0$. For $0<\delta\leq1/2$, the network classes in
Propositions~\ref{prop:stable-offset} and \ref{prop:pooled-general}
give, with probability at least $1-2\delta$,
\begin{equation}
\begin{aligned}
\calR_\ell^{\rm int}(\widehat s_\ell)
\leq C\polylog(n/\delta)\bigl[&
\underbrace{\bar\kappa n^{-2s_0/(2s_0+d)}}_{\text{pooled score}}
+\underbrace{n_\ell^{-\alpha_\Delta}}_{\text{group offset}}\bigr].
\end{aligned}
\label{eq:mixed-main}
\end{equation}
\end{theorem}
\noindent\emph{Proof idea.} We apply Theorem~\ref{thm:oracle} by separately controlling the pooled-score error and the approximation and estimation costs of the offset. For the pooled score, the Besov-based construction of \citet[Theorem~3.1]{oko2023} yields an approximation error proportional to $\sigma_t^{-2}$, whose integral contributes only a logarithmic factor by \eqref{eq:main-time-integral}. Balancing this approximation error with the network estimation cost gives the pooled rate $n^{-2s_0/(2s_0+d)}$, which transfers to the target distribution through overlap. For the offset, hierarchical ReLU approximation yields a complexity determined by the smoothness and arity of its constituents. Balancing its approximation and estimation costs gives $n_\ell^{-\alpha_\Delta}$. Substituting these two bounds into the oracle inequality proves the result.

\noindent\emph{Positive transfer.} Here Separate uses the independent adaptation sample. Suppose additionally that $p_{0,\ell}$ is supported on $[-1,1]^d$, is
bounded above and away from zero on that cube, and belongs to a fixed
restriction Besov ball $B^{s_\ell}_{p_\ell^\ast,\theta_\ell^\ast}$, where
$s_\ell>d(1/p_\ell^\ast-1/2)_+$. Assume that its boundary regularity and
the diffusion schedule satisfy the hypotheses of
\citet[Theorem~3.1]{oko2023}, uniformly over the resolutions used.
We take $n_\ell\asymp n^\rho$ for fixed $0<\rho\leq1$ and a learning
interval admissible for both the pooled and target-only approximation
resolutions; see Lemma~\ref{lem:proof-oko-input}.
We define
\(\alpha_0=\frac{2s_0}{2s_0+d},
\alpha_{\rm sep}=\frac{2s_\ell}{2s_\ell+d}.\)
We apply the argument of Proposition~\ref{prop:pooled-general} to the
target density with sample size $n_\ell$ to obtain a separate-fitting
estimator with score-risk bound
\[
\calR_\ell^{\rm int}(\widehat s_\ell^{\rm sep})
\lesssim\polylog(n/\delta)\,n_\ell^{-\alpha_{\rm sep}}
\]
with probability at least $1-\delta$, using the same learning interval
for both estimators.
The polynomial exponents in the two-stage and separate-fitting upper
bounds are therefore $\min\{\alpha_0,\rho\alpha_\Delta\}$ and
$\rho\alpha_{\rm sep}$, respectively. Thus the two-stage upper bound has a
strictly faster polynomial rate when
\begin{equation}
\alpha_0>\rho\alpha_{\rm sep},\qquad
\alpha_\Delta>\alpha_{\rm sep}.
\label{eq:positive-transfer-condition}
\end{equation}
The second inequality requires $q/K>s_\ell/d$ for every constituent
of each nonzero offset coordinate: its smoothness-to-arity ratio must
exceed the separate-fitting benchmark. The first inequality requires the
pooled score error to decay faster than the target-only bound. When $n_\ell\asymp n$ and
$\alpha_0=\alpha_{\rm sep}$, the polynomial exponents coincide; the
comparison then depends on sample-size factors. 
The following example illustrates how shared density components can cancel in the score offset, yielding a simpler compositional structure.

\begin{example}[Heterogeneity confined to one coordinate block]\label{ex:block-offset}
We consider a partition $\mathcal B$ of $[d]$ and groups sharing the same
factors except on one block $B_\star$:
\[
p_{0,g}(x)=r_{g,\star}(x_{B_\star})
\prod_{B\in\mathcal B\setminus\{B_\star\}}r_B(x_B),
\quad g\in[L].
\]
Writing $\bar r_\star=\sum_{g=1}^L\pi_g r_{g,\star}$, the pooled density is
\[
p_0(x)=\bar r_\star(x_{B_\star})
\prod_{B\in\mathcal B\setminus\{B_\star\}}r_B(x_B).
\]
The forward diffusion acts independently on each coordinate and thus preserves this block factorization:
\[
\begin{aligned}
p_{t,\ell}(x)
=
r_{\ell,\star,t}(x_{B_\star})
\prod_{B\ne B_\star} r_{B,t}(x_B),
\\
p_t(x)
=
\bar r_{\star,t}(x_{B_\star})
\prod_{B\ne B_\star} r_{B,t}(x_B).
\end{aligned}
\]
Hence the shared factors cancel exactly in the score difference, giving
\[
\begin{aligned}
\Delta_{t,\ell,B} &= 0 \quad (B\ne B_\star),\\
\Delta_{t,\ell,B_\star}
&= \nabla_{x_{B_\star}}\log
\frac{r_{\ell,\star,t}(x_{B_\star})}
{\bar r_{\star,t}(x_{B_\star})}.
\end{aligned}
\]
Thus every nonzero coordinate of the offset depends only on
$(x_{B_\star},t)$, i.e., on $|B_\star|+1$ inputs.

Suppose $\bar r_\star$ and $h_\ell=r_{\ell,\star}/\bar r_\star$ satisfy the conditions of Proposition~\ref{prop:offset}. By that proposition, each nonzero offset coordinate has joint smoothness $q=(r_{\mathrm{off},\ell}-1)/2$ after polylogarithmic normalization. Since it depends only on $(x_{B_\star},t)$, it can be represented by a single constituent of arity $|B_\star|+1$. Thus Assumption~\ref{as:stable-offset} holds with composition depth one. The pooled density is separately required to satisfy Assumption~\ref{as:pooled-besov}.

\end{example}
% The network constructions and rate proofs appear in
% Appendices~\ref{app:stable-construction}--\ref{app:mixed-main};
% Appendix~\ref{app:alternatives} treats the alternative regularity settings.

\subsection{Small-time regularity of the score offset}
\label{sec:score-properties}
\label{sec:cancellation}
A key challenge in Assumption~\ref{as:stable-offset} is its uniformity as the lower diffusion-time cutoff approaches zero. Individual diffusion scores may develop singularities in this regime, raising the question of whether their difference can remain sufficiently regular. We next establish density-level conditions under which such cancellation occurs.

Subtracting the pooled and group versions of
\eqref{eq:score-posterior} cancels the common term
$-x/\sigma_t^2$, giving
\begin{equation}
\begin{aligned}
\Delta_{t,\ell}(x)=\frac{m_t}{\sigma_t^2}\bigl\{&
\E_\ell[X_0\mid X_t=x]\\[-1mm]
&-\E[X_0\mid X_t=x]\bigr\},
\end{aligned}
\label{eq:correction-posterior}
\end{equation}
where $\E_\ell$ and $\E$ use the group and pooled priors.
Figure~\ref{fig:boundary-cancellation} illustrates this cancellation.

\begin{figure}[!htbp]
\centering
\includegraphics[width=\linewidth]{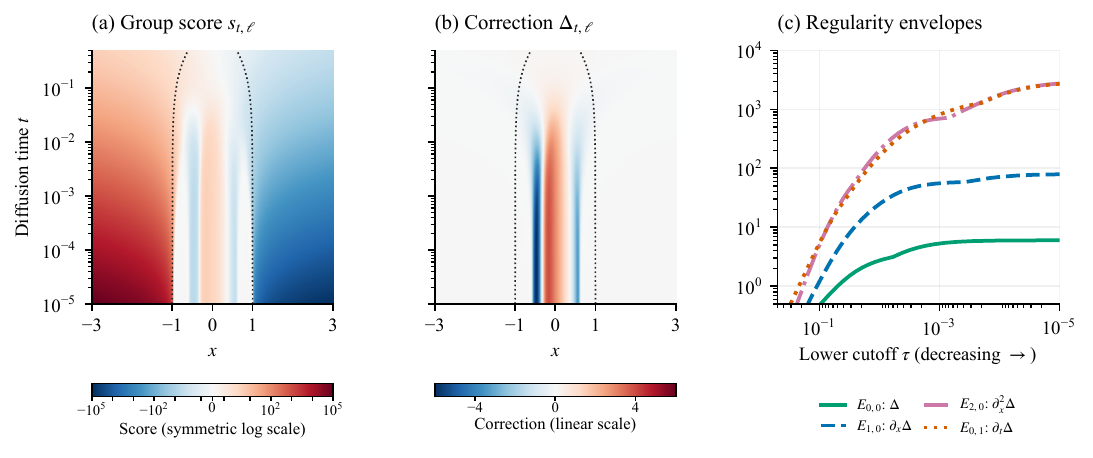}
\caption{Score cancellation and joint regularity in a smooth density-ratio example.
(a) The group score develops large boundary and exterior values as $t$ decreases.
(b) The offset remains bounded, including outside the initial support.
(c) Sampled derivative envelopes level off as the cutoff decreases.
Appendix~\ref{app:figure-construction} specifies the model, scales, and numerical checks.}
\label{fig:boundary-cancellation}
\end{figure}

An individual score can grow as $\sigma_t^{-1}$ at a smooth boundary
and as $\sigma_t^{-2}$ farther outside the support
(Appendix~\ref{app:boundary-details}). Let $r$ be a density on
$K=[-1,1]^D$, with $0<c_r\leq r\leq C_r$, and let $r_\ell=hr$
be another density, where $0<h_{\min}\leq h\leq h_{\max}$ and $h$
is Lipschitz. In rescaled coordinates $z=x/m_t$, $v=\sigma_t^2/m_t^2$, we
define, using zero extensions outside $K$,
\[
D(z,v)=(r*\varphi_{vI_D})(z),\qquad
G(z,v)=\frac{(hr)*\varphi_{vI_D}(z)}{D(z,v)}.
\]
This specializes the guidance identity of
\citet[Theorem~3.1]{ouyang2024} to Gaussian kernels.

\begin{proposition}[Uniform boundedness of the score offset]
\label{prop:ratio-cancellation}
Under these density and ratio conditions,
\begin{equation}
\Delta_{t,\ell}(x)=m_t^{-1}\nabla_z\log G(z,v).
\label{eq:general-ratio}
\end{equation}
For the posterior density
\(\bar\pi_{z,v}(u)\propto r(u)e^{-\|u-z\|_2^2/(2v)}\one_K(u),\)
we have
\[
G=\E_{\bar\pi_{z,v}}h(U),\qquad
\nabla_zG=v^{-1}\Cov_{\bar\pi_{z,v}}(h(U),U).
\]
Consequently, $h_{\min}\leq G\leq h_{\max}$ and
\(\sup_{t>0,x\in\R^D}\|\Delta_{t,\ell}(x)\|_2\leq C,\)
where $C$ depends on $D,c_r,C_r,h_{\min},h_{\max}$,
$\operatorname{Lip}(h)$, and the diffusion-schedule bounds.
\end{proposition}
\noindent\emph{Proof idea.}
We start from the variance bound $v$ for 1-Lipschitz functions under
a Gaussian conditioned on the cube and preserve this order when
reweighting by the bounded density $r$. We then apply Cauchy--Schwarz
to bound the covariance by $Cv$, cancelling $v^{-1}$; for times bounded
away from zero, we use bounded support to control posterior moments
directly.

\begin{proposition}[Joint regularity of the score offset]
\label{prop:offset}
In dimension $k$, let $\bar r$ and $h_\ell\bar r$ be densities on
$[-1,1]^k$. Suppose $\bar r$ and $h_\ell$ satisfy \(0<c_r\leq\bar r\leq C_r\) and \(0<h_{\min}\leq h_\ell\leq h_{\max}\) on \([-1,1]^k\), with fixed constants, and have fixed \(C^{r_{\mathrm{off},\ell}}([-1,1]^k)\) norms,
where $r_{\mathrm{off},\ell}=m_{\mathrm{off},\ell}+\alpha_{\mathrm{off},\ell}$,
$m_{\mathrm{off},\ell}\geq3$, and $\alpha_{\mathrm{off},\ell}\in(0,1]$.
Assume $h_\ell$ is constant on a fixed boundary collar and the schedule
has bounded derivatives through order $m_{\mathrm{off},\ell}+1$.
Uniformly over \eqref{eq:window}, on $Q_n$ (with spatial dimension $k$),
\begin{equation}
\begin{aligned}
\sup_{(x,t)\in Q_n}
\|\partial_t^b\partial_x^\gamma\Delta_{t,\ell}(x)\|_2
&\leq C\polylog(n),\\
|\gamma|_1+2b&\leq m_{\mathrm{off},\ell}-1.
\end{aligned}
\label{eq:main-mixed-correction-envelope}
\end{equation}
The same bound holds for the top-order spatial and temporal H\"older
seminorms of exponents $\alpha_{\mathrm{off},\ell}$ and $\alpha_{\mathrm{off},\ell}/2$.
After polylogarithmic normalization, the offset has a global Lipschitz,
depth-one representation of smoothness $(r_{\mathrm{off},\ell}-1)/2$.
\end{proposition}
\noindent\emph{Proof idea.}
We split space into two regions to control the derivatives. In the
interior, we use Gaussian smoothing and integration by parts to transfer
derivatives to the smooth densities and ratios. Through the heat identity
\begin{equation}
\partial_v(f*\varphi_{vI_k})=\tfrac12\Delta_z(f*\varphi_{vI_k}),
\label{eq:main-heat-identity}
\end{equation}
we express one variance derivative through two spatial derivatives.
Near and outside the boundary, we write $h_\ell=c+\psi$: the collar
condition keeps the support of $\psi$ a fixed distance inside the cube.
Its Gaussian contribution is exponentially small relative to the pooled
convolution, absorbing inverse powers of the noise scale.
Combining these two estimates, we control offset derivatives uniformly
as $t$ decreases.
Appendix~\ref{app:offset} supplies the bounds and extension argument.

\subsection{Accuracy of the generated distribution}
\label{sec:distribution}

We further translate the score estimation error rate into a distributional error bound, viewing the diffusion model as a density estimator.
For group $\ell$, we write $\gamma_d=N(0,I_d)$ and denote by
$\widehat P_\ell:=\widehat P_{\Tmin,\ell}$ the law at reverse time
$\Tmax-\Tmin$ of the continuous-time process
\begin{equation}
\begin{aligned}
\dd\widehat Y_{u,\ell}
&=\beta_{\Tmax-u}\left\{\widehat Y_{u,\ell}
+2\widehat s_\ell(\widehat Y_{u,\ell},\Tmax-u)\right\}\dd u+\sqrt{2\beta_{\Tmax-u}}\dd\overline B_u,\\
\widehat Y_{0,\ell}&\sim\gamma_d.
\end{aligned}
\label{eq:reverse-sde}
\end{equation}

We write $P_{t,\ell}$ for the law with density $p_{t,\ell}$. Under the
reverse-SDE localization conditions in Proposition~\ref{prop:tv},
we apply the decomposition of \citet[Appendix~D.1, Eq.~(78)]{oko2023}
to obtain
\begin{equation}
\begin{aligned}
\TV(P_{0,\ell},\widehat P_\ell)
\leq{}&\underbrace{\TV(P_{0,\ell},P_{\Tmin,\ell})}_{\text{early stopping}}+\underbrace{\TV(P_{\Tmax,\ell},\gamma_d)}_{\text{Gaussian initialization}}+\underbrace{\left\{\frac{\overline\beta}{2}
\calR^{\rm int}_\ell(\widehat s_\ell)\right\}^{1/2}}_{\text{score estimation}}.
\end{aligned}
\label{eq:tv-three-regimes}
\end{equation}

\begin{corollary}[Distributional rate for the two-stage estimator]
\label{cor:tv-rate}
Under Theorem~\ref{thm:mixed-main}, suppose the reverse-SDE localization
conditions hold and the zero extension of $p_{0,\ell}$ has uniformly
bounded variation. We choose admissible cutoffs so that
\[
\sqrt{\Tmin}+e^{-\underline\beta\Tmax}
\lesssim\left(n^{-2s_0/(2s_0+d)}+n_\ell^{-\alpha_\Delta}\right)^{1/2}.
\]
For any fixed $b>0$, take $\delta=n^{-b}/2$ in
Theorem~\ref{thm:mixed-main}. Then, with probability at least $1-n^{-b}$,
\begin{equation}
\TV(P_{0,\ell},\widehat P_\ell)
\leq C\polylog(n)\bigl[n^{-s_0/(2s_0+d)}
+n_\ell^{-\alpha_\Delta/2}\bigr].
\label{eq:tv-rate}
\end{equation}
\end{corollary}
\noindent\emph{Proof idea.}
We apply Girsanov's formula to control reverse-path relative entropy
by the integrated score error, then use Pinsker's inequality to obtain
TV error with half the rate exponents. We use bounded variation to
bound early-stopping error by $O(\sqrt{\Tmin})$ and OU convergence to
bound initialization error by $O(e^{-\underline\beta\Tmax})$.
By choosing cutoffs that keep both endpoint errors below the
score-estimation term, we obtain the pooled and offset rates for
generation. Appendices~\ref{app:tv}--\ref{app:tv-rate}
give the endpoint bounds and feasible cutoff choices.

\section{Experiments}
\label{sec:experiments}

We compare Offset with Pooled, Separate, output-layer Top FT~\citep{ren2026},
and an adaptation of classifier-guided TGDP~\citep{ouyang2024}, using the
baselines shown in each table. Within each experiment, methods share
data splits, diffusion schedules, and evaluation draws. Pooled, Separate,
and the narrower Offset use the same backbone family. Separate, Top FT,
and Offset have matching target-stage update counts; TGDP trains two
classifiers and a scalar potential. Appendix~\ref{app:experimental-protocols}
gives the data allocations, model sizes, and training budgets.

\subsection{Shared modes with shifted weights}
We vary the mixture weights of eight shared two-dimensional Gaussian components
to test whether the pooled model captures their common underlying geometry while the offset adjusts
their relative probability mass. With the pooled model fixed, increasing the target dataset size
improves offset estimation and broadens its performance advantage over Pooled (Table~\ref{tab:toy-results}).
Top FT has the lowest error among the methods at $n_\ell=20$.
In the no-shift control at $n_\ell=40$, Pooled outperforms Offset
($0.066$ versus $0.087$), illustrating the cost of adaptation when the
distributions coincide.

\begin{table}[!htbp]
\centering
\setlength{\belowcaptionskip}{6pt}
\caption{Eight-mode simulation: exact-score MSE averaged over ten seeds;
bold indicates the smallest unrounded mean.}
\label{tab:toy-results}
\setlength{\tabcolsep}{3pt}
\begin{tabular*}{\linewidth}{@{\extracolsep{\fill}}lrrrrr@{}}
\toprule
$n_\ell$ & Pooled & Top FT & TGDP & Separate & Offset \\
\midrule
20 & 0.333 & \textbf{0.136} & 0.150 & 0.227 & 0.172 \\
40 & 0.333 & 0.118 & 0.101 & 0.152 & \textbf{0.101} \\
80 & 0.333 & 0.105 & 0.095 & 0.120 & \textbf{0.077} \\
160 & 0.333 & 0.094 & 0.081 & 0.102 & \textbf{0.071} \\
320 & 0.333 & 0.090 & 0.078 & 0.087 & \textbf{0.056} \\
\bottomrule
\end{tabular*}
\end{table}

\subsection{A rough density with a stable offset}
\label{sec:rough-density-simulation}
To align the simulation with our theoretical analysis, we define on $[-1,1]^2$
\[
\begin{aligned}
p_0(x)&=\tfrac14[1+\tfrac12 r(x_1)b(x_2)],\\
p_{0,\ell}(x)&=p_0(x)[1+\tfrac12 A(x_1)b(x_2)].
\end{aligned}
\]
Here $r$ is a centered half-H\"older cusp supported near $-0.5$,
$A$ is a smooth odd tilt supported near $0.45$, and $b$ is a smooth
interior bump. The remaining group has density
$(p_0-\pi_\ell p_{0,\ell})/(1-\pi_\ell)$.
Separating the shared cusp from the tilt yields a rough pooled density
while preserving a stable depth-one $q=2$ offset in $(x_1,x_2,t)$; both densities
remain constant across a full boundary collar.

Under independent training stages and the learning window specified by the theory,
Offset outperforms Separate in all five seeds at each of the three target sample sizes. As $n_\ell$ increases
from 64 to 1,024, the mean effective-offset error decreases from $0.0257$ to $0.0055$, whereas
the pooled-fit error remains fixed. Performance differences relative to Pooled remain statistically inconclusive;
the empirical fits obtained via finite Adam runs serve to illustrate the construction, whereas the theorem
governs empirical minimizers over specified network function classes. Appendix~\ref{app:rough-density-protocol} reports
complete numerical results, and Proposition~\ref{prop:rough-density-witness} details the explicit construction.

\subsection{Image benchmarks}
The MNIST and Fashion-MNIST benchmarks incorporate contrast-based and class-specific target distributions
to evaluate controlled and natural distribution shifts under constrained training budgets.
Offset has the lowest mean DSM in all eight configurations
(Table~\ref{tab:image-results}) and the lowest loss in every paired seed
evaluation. Its mean loss decreases as the target sample size grows,
while the pooled fit remains fixed.

For contrast-based and truck-class targets on native CIFAR-10 RGB32
images~\citep{cifar2009,ho2020,edm2022}, Offset has the lowest mean DSM at
$n_\ell\in\{1024,2048,4096\}$ (Table~\ref{tab:cifar10-results}).
This range was selected after inspecting the initial results. Five
additional runs on a fixed full grid compare Offset with Pooled,
Separate, and TGDP, but not Top FT. They support the larger-sample
pattern, while Offset has a higher mean DSM than all three comparators
at $n_\ell=64$ and than Pooled and TGDP at $n_\ell=256$.
Appendix~\ref{app:cifar10-protocol} reports the full grid and uncertainty
for these two target settings.

\begin{table}[!htbp]
\centering
\small
\setlength{\tabcolsep}{2.3pt}
\setlength{\belowcaptionskip}{6pt}
\caption{MNIST/Fashion-MNIST ($16\times16$): mean DSM per pixel over five seeds.
The target classes are digit 8/sneaker; bold indicates the lowest mean.}
\label{tab:image-results}
\begin{tabular*}{\linewidth}{@{\extracolsep{\fill}}lrrrr@{}}
\toprule
& \multicolumn{2}{c}{Contrast} & \multicolumn{2}{c}{Class} \\
\cmidrule(lr){2-3}\cmidrule(lr){4-5}
Method/$n_\ell$ & 64 & 256 & 64 & 256 \\
\midrule
\multicolumn{5}{l}{MNIST} \\
Pooled & 0.1174 & 0.1174 & 0.1579 & 0.1579 \\
Top FT & 0.1117 & 0.1115 & 0.1548 & 0.1546 \\
TGDP & 0.1133 & 0.1133 & 0.1548 & 0.1552 \\
Separate & 0.1437 & 0.1388 & 0.1866 & 0.1787 \\
Offset & \textbf{0.1074} & \textbf{0.1070} & \textbf{0.1452} & \textbf{0.1426} \\
\midrule
\multicolumn{5}{l}{Fashion-MNIST} \\
Pooled & 0.1087 & 0.1087 & 0.1398 & 0.1398 \\
Top FT & 0.1050 & 0.1045 & 0.1324 & 0.1317 \\
TGDP & 0.1071 & 0.1071 & 0.1349 & 0.1346 \\
Separate & 0.1415 & 0.1374 & 0.1362 & 0.1298 \\
Offset & \textbf{0.0998} & \textbf{0.0984} & \textbf{0.1121} & \textbf{0.1099} \\
\bottomrule
\end{tabular*}
\end{table}

\begin{table}[!htbp]
\centering
\small
\setlength{\tabcolsep}{2.1pt}
\setlength{\belowcaptionskip}{6pt}
\caption{CIFAR-10 RGB32: mean DSM per RGB coordinate over five seeds;
bold indicates the smallest unrounded mean. These six conditions were
selected after inspecting the initial results.}
\label{tab:cifar10-results}
\begin{tabular*}{\linewidth}{@{\extracolsep{\fill}}lrrrrr@{}}
\toprule
Target/$n_\ell$ & Pooled & Top FT & TGDP & Separate & Offset \\
\midrule
Contrast/1024 & 0.0330 & 0.0323 & 0.0328 & 0.0372 & \textbf{0.0312} \\
Contrast/2048 & 0.0330 & 0.0321 & 0.0328 & 0.0371 & \textbf{0.0312} \\
Contrast/4096 & 0.0330 & 0.0324 & 0.0328 & 0.0375 & \textbf{0.0312} \\
Truck/1024 & 0.0619 & 0.0612 & 0.0618 & 0.0676 & \textbf{0.0601} \\
Truck/2048 & 0.0619 & 0.0611 & 0.0617 & 0.0670 & \textbf{0.0599} \\
Truck/4096 & 0.0619 & 0.0613 & 0.0617 & 0.0673 & \textbf{0.0598} \\
\bottomrule
\end{tabular*}
\end{table}

\section{Conclusion}
\label{sec:discussion}

Our conditional oracle bounds and small-time regularity analysis
yield convergence-rate bounds that separate pooled-score error
from offset-learning costs, identifying conditions for faster
upper-bound rates than separate fitting.

\paragraph{Limitations.}
Our theoretical analysis assumes fixed structural classes and decoupled training stages.
The derived rate comparisons rely on upper bounds; establishing corresponding lower bounds remains open
and would characterize the minimax limits of transfer learning in this context.
While we prove continuous-time generation guarantees, practical implementations introduce
discretization error during sampling. Our finite-budget empirical evaluations focus on score losses;
systematic sample-based evaluation and optimization analyses remain for future work.

Overall, our results provide a framework for understanding how
shared distributional complexity and simpler target-specific
structure can lead to improved convergence-rate bounds in
diffusion transfer.

\bibliographystyle{plainnat}
\bibliography{ref}

\clearpage
\appendix

\section{Notation and function classes}
\phantomsection
\label{app:notation}
Appendix~\ref{app:experimental-protocols} gives the experimental protocols.
Appendix~\ref{app:oracle} proves the oracle inequality, and
Appendices~\ref{app:stable-construction}--\ref{app:mixed-main}
derive the rates. Appendices~\ref{app:boundary-details} and
\ref{app:offset} establish offset regularity;
Appendix~\ref{app:rough-density-certificate} gives a compatible rough-density example.
Appendices~\ref{app:tv}--\ref{app:tv-rate} prove the distributional
guarantees, and Appendix~\ref{app:alternatives} treats alternative settings.

\paragraph{Indices, norms, and asymptotic notation.}
We write $[m]=\{1,\ldots,m\}$, $\mathbb N=\{1,2,\ldots\}$, and
$\mathbb N_0=\{0,1,\ldots\}$. For $v\in\R^D$,
$\|v\|_2=(\sum_jv_j^2)^{1/2}$ and $\|v\|_\infty=\max_j|v_j|$.
For a multi-index $\gamma\in\mathbb N_0^D$, write
$|\gamma|_1=\sum_j\gamma_j$ and
$\partial_x^\gamma=\prod_j\partial_{x_j}^{\gamma_j}$.
For $f:\R^d\times[\Tmin,\Tmax]\to\R^d$,
\[
\|f\|_{2,\infty}
=\sup_{(x,t)\in\R^d\times[\Tmin,\Tmax]}\|f(x,t)\|_2,
\qquad
\|f\|_{2,\infty;Q_n}=\sup_{(x,t)\in Q_n}\|f(x,t)\|_2.
\]
The entropy assumption uses the full-domain norm. Write
$\calN(\eta,\calF,\|\cdot\|)$ for the minimum size of an $\eta$-net.
Constants $C,c>0$ may change and depend only on fixed model parameters
unless stated otherwise. We write
$a_n\lesssim b_n$ for $a_n\leq Cb_n$ and $a_n\asymp b_n$ when both
directions hold. The notation $\polylog(u)$ denotes a fixed polynomial
in $\log u$, while $\widetilde O$ suppresses polylogarithmic factors.
For a scalar $a$, $a_+=\max\{a,0\}$, and
$\log_+(u)=\max\{\log u,0\}$ for $u>0$.

\begin{table}[htbp]
\centering
\setlength{\tabcolsep}{2.1pt}
\setlength{\belowcaptionskip}{6pt}
\caption{Principal notation.}
\label{tab:notation-guide}
\small
\begin{tabular}{p{0.26\linewidth}>{\raggedright\arraybackslash}p{0.65\linewidth}}
\toprule
Notation & Meaning \\
\midrule
$n,n_\ell$ & Pooled and independent target-group training sample sizes. \\
$\pi_\ell,p_0,p_{0,\ell}$ & Group prevalence, pooled density, and group density. \\
$m_t,\sigma_t^2$ & OU mean multiplier and noise variance. \\
$s_t,s_{t,\ell},\Delta_{t,\ell}$ & Pooled score, group score, and offset
$\Delta_{t,\ell}=s_{t,\ell}-s_t$. \\
$\widehat s,\widehat g_\ell,\widehat s_\ell$ & Pooled fit, offset fit, and
group estimator $\widehat s_\ell=\Pi(\widehat s+\widehat g_\ell)$. \\
$\xi_t(x,x_0)$ & Gaussian transition score defined in
\eqref{eq:transition-score}. \\
$Q_n,R_n,T_n$ & Effective box $[-R_n,R_n]^d\times[\Tmin,\Tmax]$,
spatial radius $R_n=1+C_R\sqrt{\log n}$, and interval length
$T_n=\Tmax-\Tmin$. \\
$K_n(t),\Pi_t$ & Clipping radius and projection onto the corresponding ball. \\
$\calR_\ell^{\rm int},\calD_{\ell,n_\ell}^{\rm int}$ & Integrated group score
risk and empirical denoising loss. \\
$r_{n,\ell}^{\rm int},\bar r_{n,\ell}$ & Pooled-score error weighted by
$p_{t,\ell}$, and its deterministic high-probability bound. \\
$\mathsf A_{n,\ell}^{\rm int},V_{n,\ell},A_{n,\ell}$ & Offset approximation
error, entropy coefficient, and entropy logarithm scale. \\
$(q,K),\mathbf M_\ell,\mathbf P_\ell$ & Constituent smoothness/arity pair,
coordinate depths, and coordinatewise sets of constituent pairs. \\
$a^\Delta_{n,\ell},\rho_m(\mathbf P_\ell),\alpha_\Delta$ & Offset
normalization envelope, structural rate, and exponent in
\eqref{eq:offset-exponent}. \\
\bottomrule
\end{tabular}
\end{table}

\paragraph{Smoothness and hierarchical composition.}
For $q=u+\alpha$ with $u\in\mathbb N_0$ and $\alpha\in(0,1]$,
$C^q=C^{u,\alpha}$ denotes the H\"older class with bounded derivatives
through order $u$ and $\alpha$-H\"older derivatives at order $u$.
Integer $q$ uses the convention $u=q-1$, $\alpha=1$.
Definition~\ref{def:general-hierarchy} also imposes global Lipschitz
regularity, including when $q<1$.

For the class in Definition~\ref{def:general-hierarchy}, define
\[
 \rho_m(\mathcal P):=\max_{(q,K)\in\mathcal P}m^{-2q/(2q+K)},
 \qquad
 \rho_m(\mathbf P):=\max_{j\in[d]}\rho_m(\mathcal P_j),
\]
ignoring identically zero coordinates and setting the rate to zero
for a class fixed to the zero function. For the target structure, the
coordinate rates are
$\rho_{n,j,\ell}:=\rho_{n_\ell}(\mathcal P_j^\ell)$, so that
$\rho_{n,\ell}=\max_j\rho_{n,j,\ell}$.

\paragraph{Restriction Besov spaces.}
For a domain $\Omega$, $B^s_{p,\theta}(\Omega)$ consists of restrictions
to $\Omega$ of functions in $B^s_{p,\theta}(\R^D)$; its restriction
norm is the infimum of the whole-space norms over extensions.
For an integer $M>s$, a whole-space norm is
\[
\|f\|_{L^p(\R^D)}+
\left(\int_0^1[u^{-s}\omega_M(f,u)_p]^\theta\,
\frac{\dd u}{u}\right)^{1/\theta},
\]
with the usual supremum for $\theta=\infty$. The modulus of smoothness is
\[
\omega_M(f,u)_p=\sup_{\|h\|_2\leq u}\|\mathsf D_h^Mf\|_{L^p},
\qquad
\mathsf D_h^Mf(x)=\sum_{j=0}^M(-1)^{M-j}\binom Mj f(x+jh).
\]
Restriction smoothness permits a jump in the zero extension.
Lemma~\ref{lem:proof-oko-input} states the additional boundary and
schedule conditions.

\paragraph{Probability distances and bounded variation.}
We use $\TV(P,Q)=\sup_A|P(A)-Q(A)|$, equal to
$\frac12\|p-q\|_1$ when densities exist, and
$\KL(P\|Q)=\int\log(\dd P/\dd Q)\dd P$ for $P\ll Q$, with value
$+\infty$ otherwise. A function $f\in L^1(\R^d)$ belongs to
$BV(\R^d)$ when its distributional gradient $Df$ is a finite
vector-valued measure; $\|f\|_{BV}=\|f\|_1+|Df|(\R^d)$.
The law $\gamma_d=N(0,I_d)$ is the standard Gaussian distribution.
\section{Experimental protocols}
\label{app:experimental-protocols}
\subsection{Eight-mode Gaussian mixture}
\label{app:simulation-protocol}
For $j=0,\ldots,7$, set $\theta_j=2\pi j/8$,
$\mu_j=2(\cos\theta_j,\sin\theta_j)$, and $\tau=0.22$.
The common and target densities are
\[
\begin{aligned}
p_{0,g}(x)&=\sum_{j=0}^7w_{g,j}\varphi_{\tau^2I_2}(x-\mu_j),\\
w_{\mathrm{common},j}&=\frac18,\qquad
w_{\mathrm{target},j}=
\frac{e^{2\cos\theta_j}}{\sum_{a=0}^7e^{2\cos\theta_a}},
\end{aligned}
\]
with centered Gaussian density $\varphi_\Sigma$. The pooled density is
$p_0=0.98p_{0,\mathrm{common}}+0.02p_{0,\mathrm{target}}$.
For OU corruption with $\beta_t=1$, let
$\upsilon_t^2=m_t^2\tau^2+\sigma_t^2$. The exact score is
\begin{equation}
\begin{aligned}
s_{t,g}(x)&=\sum_{j=0}^7\omega_{g,j}(x,t)
\frac{m_t\mu_j-x}{\upsilon_t^2},\\
\omega_{g,j}(x,t)&=
\frac{w_{g,j}\varphi_{\upsilon_t^2I_2}(x-m_t\mu_j)}
{\sum_{a=0}^7w_{g,a}\varphi_{\upsilon_t^2I_2}(x-m_t\mu_a)}.
\end{aligned}
\label{eq:toy-analytic-score}
\end{equation}
The no-shift control sets all weights to $1/8$; at $n_\ell=40$,
mean exact-score MSE is $0.066$ for Pooled and $0.087$ for Offset.

Ten runs use 6,000 pooled observations with 2\% targets.
Independent target sample sizes are $n_\ell\in\{20,40,80,160,320\}$.
Evaluation uses 4,096 independent points, uniform times in $[0.08,2]$,
and Gaussian noise. MSE sums coordinate errors.
Networks have two ReLU hidden layers of width 64 (Pooled/Separate)
or 32 (Offset), inputs $(x,t,\sin(kt),\cos(kt))$ for
$k\in\{1,2,4,8\}$, and zero-initialized outputs.
AdamW uses weight decay $0.01$ and batches $\min(256,n_\ell)$
with replacement; Table~\ref{tab:implementation-budgets} gives update counts.

\subsection{MNIST and Fashion-MNIST}
\label{app:small-image-protocol}
We use the official 60,000/10,000 training/test splits,\footnote{\raggedright Dataset sources:
\url{https://yann.lecun.org/exdb/mnist/index.html} and
\url{https://github.com/zalandoresearch/fashion-mnist}.\par}
scale images to $[-1,1]$, and resize to $16\times16$ by area interpolation.
Targets are contrast-transformed images, $x\mapsto-1+0.6(x+1)$,
or class-specific images (digit 8/sneaker).
Pools contain 6,000 images with 2\% targets; nested target sets of
64 and 256 have base-image identities disjoint from the pool.
Evaluation uses 512 official-test targets with one fixed time/noise draw
per image, shared across methods.
U-Net base widths are 8 (Pooled/Separate) and 4 (Offset);
training follows Appendix~\ref{app:implementation-summary}.

\subsection{TGDP adaptation}
\label{app:tgdp-protocol}
Our TGDP adaptation \citep{ouyang2024} uses balanced clean and
time-dependent logistic classifiers to estimate target-to-pooled density
ratios $\widehat r_0$ and $\widehat r_t$. For a scalar network $u_\psi$,
set $h_\psi(x,t)=\exp\{5\tanh(u_\psi(x,t)/5)\}$. With the pooled score
frozen, the guidance potential minimizes
\[
\begin{aligned}
\mathcal L_{\mathrm{TGDP}}(\psi)
={}&\E_{\mathrm{pooled}}\bigl[(h_\psi(X_t,t)-\widehat r_0(X_0))^2\bigr]\\
&+0.1\,\E_{\mathrm{target}}\bigl[(h_\psi(X_t,t)-\widehat r_t(X_t))^2\bigr]\\
&+\E_{\mathrm{target}}\Bigl[\bigl\|\widehat s(X_t,t)
  +\nabla_x\log h_\psi(X_t,t)
  +\varepsilon/\sigma_t\bigr\|_2^2\Bigr].
\end{aligned}
\]
In each expectation, $X_0$ is drawn uniformly with replacement from
the indicated training set, $t\sim\mathrm{Unif}[\Tmin,\Tmax]$,
$\varepsilon\sim N(0,I_d)$, and
$X_t=e^{-t}X_0+\sigma_t\varepsilon$.
Both classifier odds are clipped on the log scale to $[-5,5]$;
only the clean odds $\widehat r_0$ are normalized by their mean over
the pooled training set. The resulting score is
$\widehat s+\nabla_x\log h_\psi$.
Each stage uses AdamW with weight decay $0.01$ and the benchmark's
learning rate, batch size, and OU window
(Table~\ref{tab:implementation-budgets}).

\subsection{CIFAR-10 data, training, and evaluation}
\label{app:cifar10-protocol}
CIFAR-10 uses the official 50,000/10,000 training/test split at native
$32\times32$ RGB resolution, scaled to $[-1,1]$.\footnote{Dataset source:
\url{https://www.cs.toronto.edu/~kriz/cifar.html}.}
The contrast pool transforms 400 random images out of 20,000 by
$x\mapsto-1+0.6(x+1)$; the truck pool contains 400 trucks (class 9)
and 19,600 other images.
Nested adaptation sets of 64, 256, 1,024, 2,048, and 4,096 are disjoint
from the pool. Each target uses 1,000 official-test images:
transformed images or all test trucks.
U-Net base widths are 32 (Pooled/Separate) and 16 (Offset).

Table~\ref{tab:cifar10-results} reports the six larger-target conditions
selected after observing Offset's lowest mean.
For Table~\ref{tab:cifar10-confirmation}, both targets and all five sizes
were fixed before five additional runs comparing
Pooled, Separate, Offset, and TGDP.
Within each run, all sizes use the same pooled fit;
methods and sizes share evaluation images, times, and noise.

\begin{table}[!htbp]
\centering
\small
\setlength{\belowcaptionskip}{6pt}
\caption{Independent CIFAR-10 evaluation on the fixed target-size grid:
held-out DSM per RGB coordinate, mean $\pm$ SD over five seeds.}
\label{tab:cifar10-confirmation}
\setlength{\tabcolsep}{1.8pt}
\begin{tabular}{lrrrrr}
\toprule
Target & $n_\ell$ & Pooled & Offset & Separate & TGDP \\
\midrule
Contrast & 64 & $0.032883 \pm 0.005411$ & $0.044778 \pm 0.006910$ & $0.041606 \pm 0.007287$ & $0.032735 \pm 0.005450$ \\
Contrast & 256 & $0.032883 \pm 0.005411$ & $0.033341 \pm 0.005307$ & $0.037720 \pm 0.007181$ & $0.032709 \pm 0.005442$ \\
Contrast & 1024 & $0.032883 \pm 0.005411$ & $0.031150 \pm 0.005287$ & $0.037181 \pm 0.006231$ & $0.032711 \pm 0.005421$ \\
Contrast & 2048 & $0.032883 \pm 0.005411$ & $0.031136 \pm 0.005321$ & $0.037271 \pm 0.005668$ & $0.032716 \pm 0.005430$ \\
Contrast & 4096 & $0.032883 \pm 0.005411$ & $0.031094 \pm 0.005243$ & $0.037374 \pm 0.006194$ & $0.032704 \pm 0.005421$ \\
Truck & 64 & $0.061794 \pm 0.008361$ & $0.101629 \pm 0.013630$ & $0.089409 \pm 0.014032$ & $0.061708 \pm 0.008348$ \\
Truck & 256 & $0.061794 \pm 0.008361$ & $0.069769 \pm 0.009852$ & $0.068162 \pm 0.009917$ & $0.061702 \pm 0.008387$ \\
Truck & 1024 & $0.061794 \pm 0.008361$ & $0.059973 \pm 0.008182$ & $0.067855 \pm 0.008484$ & $0.061693 \pm 0.008372$ \\
Truck & 2048 & $0.061794 \pm 0.008361$ & $0.059715 \pm 0.008174$ & $0.067445 \pm 0.008993$ & $0.061695 \pm 0.008377$ \\
Truck & 4096 & $0.061794 \pm 0.008361$ & $0.059637 \pm 0.008186$ & $0.067224 \pm 0.009000$ & $0.061682 \pm 0.008364$ \\
\bottomrule
\end{tabular}
\end{table}

\subsection{Architectures, optimization, and loss normalization}
\label{app:implementation-summary}
Separate fits target data; Top FT updates Pooled's output layer.
Offset and TGDP freeze Pooled and fit a vector offset or two classifiers
and a scalar potential, respectively.

Image score and offset models are ReLU U-Nets with encoder widths
$(B,2B,4B)$, concatenated encoder--decoder skips, and group normalization
with at most eight groups. Each block receives sinusoidal time features.
TGDP uses three convolutional layers with ReLU and spatial averaging
for images, and two width-64 ReLU hidden layers for the Gaussian mixture.
Its time-dependent networks also receive sinusoidal time features.
Final U-Net and guidance-potential layers are zero-initialized.
Benchmark-specific base widths are given above; Tables~\ref{tab:implementation-architectures}
and~\ref{tab:implementation-budgets} give parameter counts and optimization settings.

The score and offset image models use AdamW with weight decay $10^{-4}$
and batches of at most 128 drawn with replacement; target-stage batches
have size $\min(128,n_\ell)$. Each update resamples uniform times in
$[0.05,2]$ and Gaussian noise.
We evaluate final iterates, with constant learning rates and no EMA.

\begin{table}[htbp]
\centering
\setlength{\belowcaptionskip}{6pt}
\caption{Trainable parameter counts. Offset and TGDP additionally use
a frozen pooled score; TGDP counts are per scalar network.}
\label{tab:implementation-architectures}
\small
\setlength{\tabcolsep}{4.5pt}
\begin{tabular}{lllr}
\toprule
Benchmark & Model/module & Architecture & Trainable parameters \\
\midrule
$16\times16$ images & Pooled / Separate & ReLU U-Net & 56,369 \\
$16\times16$ images & Top FT & Output convolution only & 73 \\
$16\times16$ images & Offset & ReLU U-Net & 18,073 \\
$16\times16$ images & TGDP (each scalar network) & Scalar network & 21,145 \\
CIFAR-10 RGB32 & Pooled / Separate & ReLU U-Net, 32 base channels & 800,035 \\
CIFAR-10 RGB32 & Top FT & Output convolution only & 867 \\
CIFAR-10 RGB32 & Offset & ReLU U-Net, 16 base channels & 206,787 \\
CIFAR-10 RGB32 & TGDP (each scalar network) & ReLU CNN, 24 base channels & 21,577 \\
\bottomrule
\end{tabular}
\end{table}

\begin{table}[htbp]
\centering
\setlength{\belowcaptionskip}{6pt}
\caption{Optimization settings. Target updates apply to each of Separate,
Top FT, and Offset. TGDP updates apply to each classifier and the potential.}
\label{tab:implementation-budgets}
\small
\setlength{\tabcolsep}{4pt}
\begin{tabular}{llrrrrr}
\toprule
Benchmark & Optimizer & Learning rate & Batch size & Pooled & Target & TGDP \\
\midrule
Eight-mode simulation & AdamW & $10^{-3}$ & $\leq256$ & 1,800 & 900 & 900 \\
$16\times16$ images & AdamW & $5\times10^{-4}$ & $\leq128$ & 400 & 250 & 1,000 \\
CIFAR-10 RGB32 & AdamW & $3\times10^{-4}$ & $\leq128$ & 12,000 & 3,000 & 3,000 \\
\bottomrule
\end{tabular}
\end{table}

\paragraph{Loss normalization.}
For $X_t=e^{-t}X_0+\sigma_t\varepsilon$, $\varepsilon\sim N(0,I_d)$,
DSM uses target $-\varepsilon/\sigma_t$.
By \eqref{eq:app-vincent}, the population image loss with uniform
times on $[a,b]$ is
\[
\operatorname{DSM}_{\rm avg}(f)
=\frac{\calR_\ell^{\rm int}(f)}{d(b-a)}+C_\ell,
\]
where $C_\ell$ is independent of $f$.
Thus paired DSM differences compare score accuracy.
These benchmarks use direct score outputs; the rough-density simulation
below applies radial projection.

\subsection{Rough-density simulation}
\label{app:rough-density-protocol}
We sample \eqref{eq:rough-density-witness} by rejection from uniform
$[-1,1]^2$ with density envelope $3/8$.
Each of five runs uses $n=20,000$ pooled points and nested target samples
$n_\ell\in\{64,256,1024\}$ independent of the pool.
Prevalence is $\pi=n_\ell/n$; pooled and target laws remain fixed.
Evaluation uses 256 independent target points and shared Gaussian noise.

The OU window is $[\Tmin,\Tmax]=[n^{-1/2},\log n]$.
Two-hidden-layer ReLU networks take $(x_1,x_2,t/\Tmax)$, with
widths 128 (Pooled/Separate; 17,282 parameters) and 32 (Offset; 1,250).
Direct scores include a fixed $-x$ skip; residual outputs start at zero.
Parameters are capped at $\pm n$ after each Adam update.
Scores are radially projected to radius $\sqrt{2}\sqrt{8\log n}/\sigma_t$;
Offset adds its residual to projected Pooled and projects again.
Adam uses learning rate $10^{-3}$, batch 256, and 2,000 pooled/1,000
target updates. Fresh log-uniform times and Gaussian noise receive
weight $t\log(\Tmax/\Tmin)$ for unbiased integrated DSM.
Offset and Separate share target draws.

Integrated risk is the ordinary-time integral of squared Euclidean
score error, without normalization by interval length. For the 129
geometrically spaced times
$t_j=\Tmin(\Tmax/\Tmin)^{j/128}$, $j=0,\ldots,128$,
let $\widehat e(t_j)$ denote the mean squared Euclidean score error;
exact scores use 192-point spatial quadrature. We apply the composite
trapezoidal rule in ordinary time:
\[
\widehat R=\sum_{j=0}^{127}\frac{t_{j+1}-t_j}{2}
\bigl[\widehat e(t_j)+\widehat e(t_{j+1})\bigr].
\]
Unlike the log-uniform training objective, evaluation uses no importance
weights or factor of $(\Tmax-\Tmin)^{-1}$.
Increasing quadrature order from 96 to 192 changes score coordinates
by at most $3.26\times10^{-7}$; increasing the time grid from 65 to 129
changes risks by at most $4.71\times10^{-4}$, preserving paired conclusions.

\begin{table}[!htbp]
\centering
\small
\setlength{\belowcaptionskip}{6pt}
\caption{Rough-density integrated target-score risk, mean $\pm$ SD over
five seeds. Each pooled fit is reused across sizes.}
\label{tab:rough-density-results}
\setlength{\tabcolsep}{4pt}
\begin{tabular}{lrrr}
\toprule
$n_\ell$ & Pooled & Offset & Separate \\
\midrule
64 & $0.097051 \pm 0.021734$ & $0.111041 \pm 0.017779$ & $0.179198 \pm 0.019641$ \\
256 & $0.097051 \pm 0.021734$ & $0.096461 \pm 0.016125$ & $0.155605 \pm 0.018744$ \\
1024 & $0.097051 \pm 0.021734$ & $0.093185 \pm 0.017008$ & $0.151830 \pm 0.024822$ \\
\bottomrule
\end{tabular}
\end{table}

\begin{table}[!htbp]
\centering
\small
\setlength{\belowcaptionskip}{6pt}
\caption{Offset minus control: mean paired differences and pointwise
95\% $t_4$ intervals. Negative values favor Offset.}
\label{tab:rough-density-paired}
\setlength{\tabcolsep}{3pt}
\begin{tabular}{llrr}
\toprule
$n_\ell$ & Control & Mean difference & 95\% interval \\
\midrule
64 & Pooled & $+0.013990$ & $[-0.007100,+0.035081]$ \\
64 & Separate & $-0.068156$ & $[-0.077845,-0.058468]$ \\
256 & Pooled & $-0.000589$ & $[-0.008517,+0.007339]$ \\
256 & Separate & $-0.059144$ & $[-0.069224,-0.049063]$ \\
1024 & Pooled & $-0.003866$ & $[-0.009816,+0.002085]$ \\
1024 & Separate & $-0.058645$ & $[-0.072907,-0.044383]$ \\
\bottomrule
\end{tabular}
\end{table}

Offset beats Separate in all 15 paired fits; the methods differ in capacity.
All intervals against Pooled include zero, although Offset wins each seed
at $n_\ell=1024$.
For the effective offset $\widehat g_\ell^{\rm eff}=\widehat s_\ell-\widehat s$,
including projection, integrated errors against $\Delta_\ell$ average
$0.025729$, $0.008807$, and $0.005504$ at $n_\ell=64,256,1024$.
The mean inherited error $\int\|\widehat s-s_t\|^2p_{t,\ell}\,dx\,dt$
is $0.096464$; target risk additionally includes offset error and their
cross term. The true offset's squared integrated signal is $0.0001514$.

\section{Proof of the two-stage oracle inequality}
\label{app:oracle}

We first control the clipped loss and its tails, then prove an
empirical-risk oracle inequality and insert the pooled-plus-offset
comparator. Throughout, we condition on $\widehat s$ and $n_\ell$;
the target observations remain i.i.d. from $p_{0,\ell}$ and the
candidate class is fixed. Write
\[
\ell_f(x,x_0,t)=\norm{f(x,t)-\xi_t(x,x_0)}_2^2,
\qquad \xi_t(x,x_0)=-\frac{x-m_tx_0}{\sigma_t^2},
\]
and let $D(f)=\int_{\Tmin}^{\Tmax}\E\ell_f(X_t,X_0,t)\dd t$,
$D_n(f)=\calD^{\rm int}_{\ell,n_\ell}(f)$ and
$R(f)=\calR^{\rm int}_\ell(f)$. Vincent's identity
$\E[\xi_t\mid X_t]=s_{t,\ell}(X_t)$~\citep{vincent2011} gives
\begin{equation}
D(f)-D(f')=R(f)-R(f')\qquad\text{for all }f,f'.
\label{eq:app-vincent}
\end{equation}

\subsection{Integrated envelopes, tails, and clipping}

\begin{lemma}[Integrated OU envelopes]
\label{lem:app-envelope}
There are constants $c_\beta,C_\beta>0$, depending only on
$(\underline\beta,\overline\beta)$, such that
\[
c_\beta(t\wedge1)\leq\sigma_t^2\leq C_\beta(t\wedge1).
\]
Consequently, under \eqref{eq:window},
\begin{align}
\Lambda_2&:=\int_{\Tmin}^{\Tmax}\sigma_t^{-2}\dd t
\leq C\{\log(1/\Tmin)+\Tmax\}\leq C\log n,
\label{eq:app-lambda2}\\
\Lambda_1&:=\int_{\Tmin}^{\Tmax}\sigma_t^{-1}\dd t
\leq C(1+\Tmax)\leq C\log n.
\label{eq:app-lambda1}
\end{align}
In particular,
\begin{equation}
\mathcal C_n:=4dC_n^2\Lambda_2\leq C d\log^2n.
\label{eq:app-Cn}
\end{equation}
\end{lemma}

\begin{proof}
The schedule bounds and
$\sigma_t^2=1-\exp\{-2\int_0^t\beta_u\dd u\}$ give
$\sigma_t^2\asymp t$ for $t\leq1$ and $\sigma_t^2\asymp1$ for
$t\geq1$. Splitting each integral at one proves the bounds on
$\Lambda_1,\Lambda_2$; \eqref{eq:window} and
$C_n^2=2(a+2)\log n$ then give \eqref{eq:app-Cn}.
\end{proof}

We define the effective event at time $t$ by
\[
E_t:=\left\{\norm{(X_t-m_tX_0)/\sigma_t}_\infty\leq C_n\right\},
\qquad \ell_f^E:=\ell_f\one_{E_t}.
\]

\begin{lemma}[Truncated loss, Gaussian tails, and target clipping]
\label{lem:app-tail}
Let $\mathcal F_\ell(\widehat s)=
\{\Pi(\widehat s+g):g\in\calG_\ell\}$.  Uniformly over
$f\in\mathcal F_\ell(\widehat s)$,
\begin{equation}
0\leq\ell_f^E(x,x_0,t)\leq\frac{4dC_n^2}{\sigma_t^2},
\qquad
\int_{\Tmin}^{\Tmax}\E[\ell_f^E\mid X_0=x_0]\dd t\leq\mathcal C_n.
\label{eq:app-loss-envelope}
\end{equation}
Moreover, for a numerical constant $C$,
\begin{equation}
\tau^{\rm int}_n:=C d^2(1+C_n^2)\Lambda_2e^{-C_n^2/2}
\leq\operatorname{poly}(d,\log n)n^{-(a+2)}
\label{eq:app-tail-def}
\end{equation}
satisfies, both for the population loss and deterministically for the empirical loss,
\begin{equation}
\sup_{f\in\mathcal F_\ell(\widehat s)}
\int_{\Tmin}^{\Tmax}\E[\ell_f\one_{E_t^c}\mid X_0]\dd t\leq\tau^{\rm int}_n.
\label{eq:app-tail-loss}
\end{equation}
Finally,
\begin{equation}
\int_{\Tmin}^{\Tmax}\E_{p_{t,\ell}}
\norm{s_{t,\ell}(X_t)-\Pi_ts_{t,\ell}(X_t)}_2^2\dd t\leq\tau^{\rm int}_n.
\label{eq:app-target-clip}
\end{equation}
\end{lemma}

\begin{proof}
Conditional on $X_0=x_0$, write $X_t=m_tx_0+\sigma_tZ$ with
$Z\sim N(0,I_d)$. On $E_t$, both $\norm f_2$ and
$\norm{\xi_t}_2$ are at most $K_n(t)$, proving
\eqref{eq:app-loss-envelope}. Everywhere,
$\ell_f\leq2\sigma_t^{-2}(dC_n^2+\norm Z_2^2)$, so Gaussian
tails and truncated second moments~\citep[Chapter~2]{boucheron2013} give
\[
\E[\ell_f\one_{E_t^c}\mid X_0=x_0]
\leq Cd^2(1+C_n^2)\sigma_t^{-2}e^{-C_n^2/2}.
\]
This is uniform in $x_0,f$. Integrating and averaging over either
the population or the observed clean samples proves
\eqref{eq:app-tail-loss}. Since $e^{-C_n^2/2}=n^{-(a+2)}$,
Lemma~\ref{lem:app-envelope} gives \eqref{eq:app-tail-def}.

For target clipping, put $V(X_t)=\E[\norm Z_\infty\mid X_t]$.
The denoising identity gives
$\norm{s_{t,\ell}(X_t)}_2\leq\sqrt dV(X_t)/\sigma_t$.
Conditional Jensen yields
$\E e^{\lambda V}\leq\E e^{\lambda\norm Z_\infty}
\leq2de^{\lambda^2/2}$ for $\lambda\geq0$.
The projection residual formula and integration of this tail bound imply
\[
\E\norm{s_{t,\ell}-\Pi_ts_{t,\ell}}_2^2
\leq\frac d{\sigma_t^2}\E[V^2\one\{V>C_n\}]
\leq Cd^2(1+C_n^2)\sigma_t^{-2}e^{-C_n^2/2}.
\]
Its time integral proves \eqref{eq:app-target-clip}.
\end{proof}

Let $D^E,D_n^E$ denote the losses formed from $\ell_f^E$.
Lemma~\ref{lem:app-tail} and \eqref{eq:app-vincent} give
\begin{align}
\sup_f\bigl(|D(f)-D^E(f)|\vee|D_n(f)-D_n^E(f)|\bigr)
&\leq\tau^{\rm int}_n,
\label{eq:app-tail-comparison}\\
\left|\{D^E(f)-D^E(f')\}-\{R(f)-R(f')\}\right|
&\leq2\tau^{\rm int}_n.
\label{eq:app-risk-comparison}
\end{align}

\subsection{A finite-net oracle argument}

\begin{lemma}[Clipped empirical-risk oracle bound]
\label{lem:app-generic-oracle}
For $0<\delta\leq1/2$, conditional on $\widehat s$ and $n_\ell$,
with probability at least $1-\delta$,
\begin{equation}
\calR^{\rm int}_\ell(\widehat s_\ell)
\leq C\left[
\inf_{f\in\mathcal F_\ell(\widehat s)}\calR^{\rm int}_\ell(f)
+\mathcal C_n
\frac{V_{n,\ell}\log(A_{n,\ell}n_\ell K_n^{\max})+\log(1/\delta)}{n_\ell}
+\tau^{\rm int}_n\right].
\label{eq:app-generic-oracle}
\end{equation}
\end{lemma}

\begin{proof}
Write $\mathcal F=\mathcal F_\ell(\widehat s)$ and set
\[
\psi=c_0\mathcal C_n
\frac{V_{n,\ell}\log(A_{n,\ell}n_\ell K_n^{\max})+\log(1/\delta)}{n_\ell},
\qquad \eta=(n_\ell K_n^{\max})^{-1},
\]
where $c_0$ is sufficiently large. Choose $f^\dagger\in\mathcal F$
with $R(f^\dagger)\leq\inf_{f\in\mathcal F}R(f)+\psi$.

\emph{Finite cover.}
Projection nonexpansiveness transfers an $\eta$-net of
$\calG_\ell$ to an $\eta$-net $\mathcal F_\eta$ of $\mathcal F$ with
\begin{equation}
\log|\mathcal F_\eta|\leq
V_{n,\ell}\log(A_{n,\ell}n_\ell K_n^{\max}).
\label{eq:app-net-size}
\end{equation}
For each $f$, choose a representative $\pi(f)$ in this net.
Expanding the squared losses on $E_t$ gives
\begin{equation}
|\ell_f^E-\ell_{f'}^E|
=|\langle f-f',f+f'-2\xi_t\rangle|\one_{E_t}
\leq\frac{4\sqrt dC_n}{\sigma_t}\norm{f-f'}_2.
\label{eq:app-loss-lip}
\end{equation}
After conditional integration and averaging,
\begin{equation}
|D^E(f)-D^E(\pi(f))|\vee|D_n^E(f)-D_n^E(\pi(f))|
\leq\zeta:=4\eta\sqrt dC_n\Lambda_1
\leq\frac{2\mathcal C_n}{n_\ell}.
\label{eq:app-net-transfer}
\end{equation}
The last bound uses $\sigma_-\Lambda_1\leq\Lambda_2$ and
$dC_n^2\geq1$. Since $\delta\leq1/2$, increasing $c_0$ ensures
$\zeta\leq\psi/2$.

\emph{Concentration.}
For $f\in\mathcal F_\eta$, define the i.i.d. variables
\[
W_{f,i}=\int_{\Tmin}^{\Tmax}
\E[\ell_f^E-\ell_{f^\dagger}^E\mid X_{\ell,i}]\dd t,
\qquad \mu_f=\E W_{f,i}=D^E(f)-D^E(f^\dagger).
\]
The losses here are evaluated at $(X_t,X_{\ell,i},t)$.
Lemma~\ref{lem:app-tail} gives $|W_{f,i}|\leq\mathcal C_n$.
To control their variance, put
$A_f(x_0,t)=\E[\norm{f(X_t,t)-f^\dagger(X_t,t)}_2^2\mid X_0=x_0]$.
Applying conditional and then temporal Cauchy--Schwarz to
\eqref{eq:app-loss-lip} yields
\[
W_{f,i}^2
\leq\left(\int_{\Tmin}^{\Tmax}4K_n(t)A_f(X_{\ell,i},t)^{1/2}\dd t\right)^2
\leq4\mathcal C_n\int_{\Tmin}^{\Tmax}A_f(X_{\ell,i},t)\dd t.
\]
Averaging and inserting $s_{t,\ell}$ between the candidates gives
\begin{align}
\Var(W_{f,i})
&\leq4\mathcal C_n\int_{\Tmin}^{\Tmax}
\E_{p_{t,\ell}}\norm{f-f^\dagger}_2^2\dd t\notag\\
&\leq8\mathcal C_n\{R(f)+R(f^\dagger)\}.
\label{eq:app-variance}
\end{align}
Consider a net element satisfying
\begin{equation}
R(f)\geq8R(f^\dagger)+8\psi+C_1\tau^{\rm int}_n,
\qquad C_1\geq16.
\label{eq:app-large-risk}
\end{equation}
Then \eqref{eq:app-risk-comparison} and \eqref{eq:app-variance}
give $\mu_f\geq3R(f)/4$ and
$\Var(W_{f,i})\leq9\mathcal C_nR(f)$.
Moreover, $2\tau_n^{\rm int}+\zeta\leq\mu_f/4$ and the centered
summands are bounded by $2\mathcal C_n$. Bernstein's
inequality~\citep[Chapter~2]{boucheron2013} therefore gives
\begin{equation}
\Pp\!\left(D_n^E(f)\leq D_n^E(f^\dagger)+2\tau^{\rm int}_n+\zeta\right)
\leq\exp\left\{-c\frac{n_\ell R(f)}{\mathcal C_n}\right\}.
\label{eq:app-bernstein}
\end{equation}
Since $R(f)\geq8\psi$, choosing $c_0$ sufficiently large and using
\eqref{eq:app-net-size} excludes every such comparison simultaneously
with probability at least $1-\delta$.

\emph{Transfer to the minimizer.}
On that event, empirical minimization,
\eqref{eq:app-tail-comparison}, and \eqref{eq:app-net-transfer} give
\[
D_n^E(\pi(\widehat s_\ell))
\leq D_n^E(f^\dagger)+2\tau_n^{\rm int}+\zeta.
\]
Thus $\pi(\widehat s_\ell)$ fails \eqref{eq:app-large-risk}.
The population versions of the same bounds give
$R(\widehat s_\ell)\leq R(\pi(\widehat s_\ell))+\zeta+2\tau_n^{\rm int}$,
and hence
\[
R(\widehat s_\ell)
\leq8R(f^\dagger)+8\psi+C\tau_n^{\rm int}+\zeta
\leq8\inf_{f\in\mathcal F}R(f)+C\psi+C\tau_n^{\rm int}.
\]
Substituting $\psi$ proves \eqref{eq:app-generic-oracle}.
\end{proof}

The proof uses only clipping, the entropy bound, and i.i.d. clean
observations. Lemmas~\ref{lem:app-tail} and
\ref{lem:app-generic-oracle} therefore also apply to any fixed clipped
score class, including the pooled class with sample size $n$.

\paragraph{Approximate minimization and sampled corruptions.}
An empirical tolerance $\varepsilon$ adds $C\varepsilon$ to
\eqref{eq:app-generic-oracle}: replace $\psi$ by $\psi+\varepsilon$
in the proof. In particular, on an event where
\[
\sup_{f\in\mathcal F}|D_n^{\rm MC}(f)-D_n(f)|\leq\varepsilon_{\rm MC},
\qquad D_n^{\rm MC}(\widetilde f)
\leq\inf_{f\in\mathcal F}D_n^{\rm MC}(f)+\varepsilon_{\rm opt},
\]
one has $D_n(\widetilde f)\leq\inf_{f\in\mathcal F}D_n(f)
+2\varepsilon_{\rm MC}+\varepsilon_{\rm opt}$.
Each stage may therefore carry this additional risk term; the pooled
stage's term transfers to the target law through overlap.

\begin{lemma}[Clipped two-stage comparator]
\label{lem:app-clipped-comparator}
On the first-stage event in Assumption~\ref{as:oracle},
\begin{equation}
\inf_{f\in\mathcal F_\ell(\widehat s)}\calR^{\rm int}_\ell(f)
\leq4\bar r_{n,\ell}+4\mathsf A^{\rm int}_{n,\ell}+2\tau^{\rm int}_n.
\label{eq:app-clipped-comparator}
\end{equation}
\end{lemma}

\begin{proof}
For any field $f$, projection nonexpansiveness gives
\[
\norm{\Pi_tf-s_{t,\ell}}_2
\leq\norm{f-s_{t,\ell}}_2+\norm{\Pi_ts_{t,\ell}-s_{t,\ell}}_2.
\]
Squaring, integrating, and applying \eqref{eq:app-target-clip} yields
\begin{equation}
\calR_\ell^{\rm int}(\Pi f)
\leq2\calR_\ell^{\rm int}(f)+2\tau_n^{\rm int}.
\label{eq:app-projection-risk}
\end{equation}
For $f=\widehat s+g$, the identity
$f-s_\ell=(\widehat s-s)+(g-\Delta_\ell)$ similarly gives
\[
\calR_\ell^{\rm int}(f)
\leq2\bar r_{n,\ell}
+2\int_{\Tmin}^{\Tmax}\!\int\norm{g-\Delta_{t,\ell}}_2^2
p_{t,\ell}\dd x\dd t.
\]
Taking the infimum over $g\in\calG_\ell$ proves the result.
\end{proof}

\begin{proof}[Proof of Theorem~\ref{thm:oracle}]
Insert Lemma~\ref{lem:app-clipped-comparator} into
Lemma~\ref{lem:app-generic-oracle} and use
$\mathcal C_n\leq Cd\log^2n$ from Lemma~\ref{lem:app-envelope}.
Absorbing numerical constants and the Gaussian remainders into
$\tau_{n,\ell}^{\rm int}$ proves \eqref{eq:oracle} with the stated
conditional probability. Finally,
$\sigma_-^2\geq cn^{-\kappa_0}$ implies
$K_n^{\max}\leq C\sqrt{d\log n}\,n^{\kappa_0/2}$,
so the maximum clipping radius appears only logarithmically.
\end{proof}
\section{Details of boundary cancellation}
\label{app:boundary-details}

\begin{proposition}[Posterior identities and boundary growth]
\label{prop:score-boundary}
For either the pooled or a group distribution with finite second moment,
\begin{align}
s_t(x)&=-\frac{x}{\sigma_t^2}
+\frac{m_t}{\sigma_t^2}\E[X_0\mid X_t=x],\label{eq:score-posterior}\\
\nabla s_t(x)&=-\frac{I_d}{\sigma_t^2}
+\frac{m_t^2}{\sigma_t^4}\Cov(X_0\mid X_t=x),\label{eq:score-jacobian}
\end{align}
and $\E\|s_t(X_t)\|_2^2\leq d/\sigma_t^2$.
Suppose also that $p_0$ is supported on the closure of a bounded domain
$\Omega$, whose boundary is $C^1$ near $u_0$, with outward unit normal
$\nu(u_0)$. If $p_0(u)\to r_*>0$ as $\Omega\ni u\to u_0$, then
\begin{equation}
\sigma_t s_t(m_tu_0)\longrightarrow
-\sqrt{2/\pi}\,\nu(u_0)\qquad(t\downarrow0).
\label{eq:boundary-asymptotic}
\end{equation}
\end{proposition}

\begin{proposition}[Flat ratios suppress boundary derivatives]
\label{prop:collar-cancellation}
Under the density and ratio conditions of Proposition~\ref{prop:ratio-cancellation},
suppose further that $h=c_\partial$ on a boundary collar of width
$a_0\in(0,1)$. Put $K_{a_0}=[-1+a_0,1-a_0]^D$. Then
\begin{equation}
\begin{aligned}
G&=c_\partial+N/D,\\
N&=((h-c_\partial)r)*\varphi_{vI_D},\\
\operatorname{supp}((h-c_\partial)r)&\subset K_{a_0}.
\end{aligned}
\label{eq:general-boundary-ratio}
\end{equation}
For every fixed multi-index $\gamma$ and integer $b\geq0$ with
$|\gamma|+b>0$, there are constants $C_{\gamma,b},c_{\gamma,b}>0$
such that, for small $v>0$,
\[
\sup_{\operatorname{dist}(z,K_{a_0})\geq a_0/4}
|\partial_z^\gamma\partial_v^b\log G(z,v)|
\leq C_{\gamma,b}e^{-c_{\gamma,b}/v}.
\]
The same bound holds on $\operatorname{dist}(z,S)\geq a>0$, where
$S=\operatorname{supp}((h-c_\partial)r)$ and $a$ is fixed, with constants
also depending on $a$.
\end{proposition}

\subsection{Gaussian localization and quotient differentiation}

The following estimates control the denominator as well as the localized
perturbation. Constants may depend on dimension, collar width, density
bounds, and differentiation order.

\begin{lemma}[Mixed derivatives of the Gaussian kernel]
\label{lem:boundary-kernel}
For fixed $\gamma$ and $b\geq0$, the derivative
$\partial_z^\gamma\partial_v^b\varphi_{vI_D}(z-u)$ is the Gaussian kernel
multiplied by a polynomial in $z-u$ and $v^{-1}$. Consequently, for
$0<v\leq v_0$ and some fixed integer $M$,
\[
\bigl|\partial_z^\gamma\partial_v^b\varphi_{vI_D}(z-u)\bigr|
\leq Cv^{-M}(1+\|z-u\|_2)^M\varphi_{vI_D}(z-u).
\]
If a bounded integrand is supported a fixed distance $a>0$ from $z$,
its convolution and every fixed mixed derivative are bounded by
$Cv^{-M'}e^{-ca^2/v}$.
\end{lemma}

\begin{proof}
Differentiating $(2\pi v)^{-D/2}e^{-\|z-u\|_2^2/(2v)}$ inductively proves
the polynomial form. For localization, split the Gaussian exponential
into two equal factors: one is at most $e^{-a^2/(4v)}$, and integrating
the other against any fixed polynomial contributes only powers of $v$
and $v^{-1}$.
\end{proof}

\begin{lemma}[Posterior mass of the inner cube]
\label{lem:boundary-cube-mass}
Let $K=[-1,1]^D$, $K^\circ=[-1+a_0,1-a_0]^D$, and
$0<c_r\leq r\leq C_r$ on $K$. Define
\[
\bar\pi_{z,v}(\dd u)
=\frac{r(u)e^{-\|u-z\|_2^2/(2v)}\one_K(u)\dd u}
       {\int_Kr(w)e^{-\|w-z\|_2^2/(2v)}\dd w}.
\]
For small $v>0$, uniformly on $\operatorname{dist}(z,K^\circ)\geq a_0/4$,
\[
\bar\pi_{z,v}(K^\circ)
\leq C\exp\left\{-\frac{c(a_0^2+a_0\operatorname{dist}(z,K))}{v}\right\}.
\]
The same estimate holds for $\bar\pi_{z,v}(S)$ when $S\subset K^\circ$
and $\operatorname{dist}(z,S)\geq a>0$, with constants also depending on $a$.
\end{lemma}

\begin{proof}
Write $r_z=\operatorname{dist}(z,K)$. On the first region, coordinatewise
projection gives
\[
\operatorname{dist}(z,K^\circ)^2-r_z^2
=\sum_i\bigl[(|z_i|-1+a_0)_+^2-(|z_i|-1)_+^2\bigr]
\geq c_D(a_0^2+a_0r_z).
\]
Inside $K$, the assumed separation supplies $a_0^2/16$; outside $K$,
each exterior coordinate contributes $2a_0(|z_i|-1)+a_0^2$.
For $S\subset K^\circ$, the same argument applies outside $K$,
while inside $K$ its separation supplies $a^2$.

At a nearest point of $K$ to $z$, an inward box of side
$\epsilon=v/[C_D(1+r_z)]$ lies in $K$ and satisfies
$\|u-z\|_2^2\leq r_z^2+C_Dv$. Thus
\[
\int_Kr(u)e^{-\|u-z\|_2^2/(2v)}\dd u
\geq c\epsilon^D e^{-r_z^2/(2v)-C_D}.
\]
Bounding the numerator by its maximum integrand times the cube's volume
and dividing by this lower bound gives
\[
\bar\pi_{z,v}(S)\leq Cv^{-D}(1+r_z)^D e^{-c(a_0^2+a_0r_z)/v},
\]
with $S=K^\circ$ for the first assertion.
A fixed fraction of the exponential absorbs the polynomial uniformly in $z$.
\end{proof}

\subsection{Posterior identities and the boundary limit}

\begin{proof}[Proof of Proposition~\ref{prop:score-boundary}]
\emph{Posterior identities.}
Since $\nabla_xq_t(x\mid u)=-(x-m_tu)q_t(x\mid u)/\sigma_t^2$,
differentiation under $p_t(x)=\int q_t(x\mid u)p_0(u)\dd u$ gives
\eqref{eq:score-posterior}. The quotient rule gives
\[
\partial_{x_j}\E[X_{0,i}\mid X_t=x]
=\frac{m_t}{\sigma_t^2}\Cov(X_{0,i},X_{0,j}\mid X_t=x),
\]
which proves \eqref{eq:score-jacobian}.
Writing $X_t=m_tX_0+\sigma_tZ$ with independent $Z\sim N(0,I_d)$,
we also have $s_t(X_t)=\E[-Z/\sigma_t\mid X_t]$.
Conditional Jensen therefore gives
$\E\|s_t(X_t)\|_2^2\leq\sigma_t^{-2}\E\|Z\|_2^2=d/\sigma_t^2$.

\emph{Boundary limit.}
Extend $p_0$ by zero and put $v=\sigma_t^2/m_t^2$ and
$g_v=p_0*\varphi_{vI_d}$. The $C^1$ boundary rescales to
$H_-:=\{z:z\cdot\nu(u_0)<0\}$. Changing variables gives
\begin{align*}
g_v(u_0)
&=\int p_0(u_0+\sqrt v\,z)\varphi_{I_d}(z)\dd z\longrightarrow r_*/2,\\
\sqrt v\,\nabla g_v(u_0)
&=\int z\,p_0(u_0+\sqrt v\,z)\varphi_{I_d}(z)\dd z
\longrightarrow-\frac{r_*}{\sqrt{2\pi}}\nu(u_0).
\end{align*}
The trace assumption permits dominated convergence locally; outside a
fixed neighborhood, the Gaussian kernel bounds give
$Cv^{-M}e^{-c/v}\|p_0\|_1$. The limits are the mass and first moment
of a standard Gaussian on $H_-$.
Finally, $s_t(m_tu_0)=m_t^{-1}\nabla\log g_v(u_0)$ and
$\sigma_t/m_t=\sqrt v$, proving \eqref{eq:boundary-asymptotic}.
\end{proof}

\subsection{Cancellation in the boundary and exterior zones}

\begin{proof}[Proof of Proposition~\ref{prop:collar-cancellation}]
The common factor $m_t^{-D}$ cancels from the diffused density ratio,
so $G=\E_{\bar\pi_{z,v}}h$ lies in $[h_{\min},h_{\max}]$ and
\eqref{eq:general-boundary-ratio} holds.
Let $S=\operatorname{supp}((h-c_\partial)r)\subset K^\circ$,
$r_z=\operatorname{dist}(z,K)$, and let $\lambda=(\gamma,b)$ denote
a mixed spatial--variance derivative. Since
$\|z-u\|_2\leq r_z+2\sqrt D$ for $u\in K$,
Lemma~\ref{lem:boundary-kernel} gives
\[
\left|\frac{\partial^\lambda D}{D}\right|
\leq Cv^{-M}(1+r_z)^M,\qquad
\left|\frac{\partial^\lambda N}{D}\right|
\leq Cv^{-M}(1+r_z)^M\bar\pi_{z,v}(S).
\]
Lemma~\ref{lem:boundary-cube-mass} bounds the posterior mass on either
region in the proposition. For $H=N/D$, Leibniz's rule gives
\[
\partial^\lambda H=\frac{\partial^\lambda N}{D}
-\sum_{0<\mu\leq\lambda}\binom{\lambda}{\mu}
\frac{\partial^\mu D}{D}\partial^{\lambda-\mu}H.
\]
Induction preserves the exponential factor. Since $G=c_\partial+H\geq h_{\min}$,
every positive-order derivative of $\log G$ is a sum of products of
positive-order derivatives of $H$ divided by powers of $G$. Hence
\begin{equation}
|\partial_z^\gamma\partial_v^b\log G(z,v)|
\leq Cv^{-M}(1+r_z)^M
 e^{-c(a_0^2+a_0r_z)/v},\qquad |\gamma|+b>0.
\label{eq:boundary-weighted-cancellation}
\end{equation}
The exponential absorbs any fixed powers of $v^{-1}$ and $z$, uniformly
in the full exterior. This proves both assertions and supplies the
weighted estimate used below.
\end{proof}

\subsection{Construction of Figure~\ref{fig:boundary-cancellation}}
\label{app:figure-construction}
The illustration uses $p_0(u)=\tfrac12\one_{[-1,1]}(u)$ and
$p_{0,\ell}=hp_0$, where
\[
\begin{aligned}
h(u)&=1+0.65\,b(u/0.65)/c_b,\\
b(y)&=\begin{cases}
y\exp\{1-1/(1-y^2)\},&|y|<1,\\
0,&|y|\geq1,
\end{cases}\\
c_b&=y_\star\exp\{1-1/(1-y_\star^2)\},\qquad
y_\star=\sqrt{2-\sqrt3}.
\end{aligned}
\]
Oddness of $b$ gives $\int hp_0=1$; the ratio lies in $[0.35,1.65]$
and equals one on collars of width $0.35$.
With $\beta_t=1$, $m_t=e^{-t}$ and $v_t=e^{2t}-1$.
We evaluate 1,201 equally spaced spatial points on $[-3,3]$ and 181
logarithmically spaced times on $[10^{-5},0.5]$, using 1,024-point
Gauss--Legendre quadrature on $[-1,1]$ with exponential rescaling in the tails.
Spatial derivatives are transferred to the bump, and time derivatives
use the heat identity and OU chain rule.
Panel (c) reports sampled spatial and cumulative time maxima over
$t\geq\tau$ for $\Delta$, $\partial_x\Delta$, $\partial_x^2\Delta$,
and $\partial_t\Delta$.
At 23 validation times, doubling quadrature order changed every field by
less than $10^{-8}$ of its maximum magnitude; doubling spatial resolution
changed the sampled maxima by less than $1.2\%$ of the respective field
maxima. These sampled envelopes illustrate the analytic bounds of
Proposition~\ref{prop:offset}.
\section{Construction of the stable offset network}
\label{app:stable-construction}

We construct an approximant and bound the entropy of the entire fitted
class. The coefficient bounds ensure that both use the same sieve.
Fix a tail exponent $A>0$ and choose $C_R$ so that
$C_R^2/2>A+2\kappa_0+1$; Assumption~\ref{as:stable-offset} is
imposed on this choice of $Q_n$, which also fixes the input clamp below.

\subsection{Imported approximation bounds}
For a globally smooth Lipschitz constituent of arity $K$, smoothness
$q\geq1$, and regularity bound $C$, the growing-width construction
of \citet[Theorem~2(a)]{kohlerlanger2020} has error
$C(a+C+1)^\gamma M^{-2q}$ on $[-a,a]^K$, depth $O(\log M)$,
and width $O(M^K)$. Thus its dense parameter count is
$O(M^{2K}\log M)$. For $N\leq n^c$ and $a=O(\log n)$, choose
$M=\lceil N^{1/(2K)}\rceil\polylog(n)$, with the logarithmic
factor large enough to meet the minimum-resolution condition. Then
\begin{equation}
\epsilon^2\leq C\polylog(n)N^{-2q/K},
\qquad W\leq CN\polylog(n).
\label{eq:constituent-approximation}
\end{equation}
For $q<1$, the separately imposed Lipschitz condition permits the
$q=1$ construction, whose error is smaller at the same budget.
Fixed hierarchical compositions preserve these rates by the Lipschitz
error propagation of \citet[Theorem~3(a)]{kohlerlanger2020}.

For clamped networks with $W$ nonzero coefficients and the budgets below,
bounded-parameter quantization~\citep[Lemma~5]{schmidthieber2020} gives
\begin{equation}
\log\calN(\eta,\mathcal F,\|\cdot\|_{2,\infty})
\leq CW\polylog(n)\log(n^C/\eta).
\label{eq:network-entropy}
\end{equation}
We record the parameter bounds and quantization argument next.

\subsection{Fitted network classes and coefficient bounds}
\label{app:fitted-sieves}
Let $\Phi_{D,d}(L,\nu,S,B)$ contain ReLU networks with $D$ inputs,
$d$ outputs, $L$ hidden layers of width $\nu$, at most $S$ nonzero
affine coefficients, and all weights and biases in $[-B,B]$.
Smaller networks are padded with zeros and signed identity channels.
Writing $P(L,\nu)$ for the number of coefficient positions,
$S=P(L,\nu)$ gives the dense class. Define the fixed input and output clamps
\[
c_n(x,t)=\bigl([x]_{[-R_n,R_n]^d},[t]_{[\Tmin,\Tmax]}\bigr),
\qquad \Gamma_H(v)=[v]_{[-H,H]^d}.
\]
For $m=n_\ell$, set
\[
N_*=\max_{j,(q,K)\in\mathcal P_j^\ell}
\left\lceil m^{K/(2q+K)}\right\rceil,
\]
omitting zero coordinates. With sufficiently large fixed constants, use
\[
L_\Delta=\lceil C_L\log(en)\rceil,\quad
\nu_\Delta=\left\lceil C_\nu N_*^{1/2}(\log(en))^{c_\nu}\right\rceil,\quad
B_\Delta=\exp\{C_B(\log(en))^2\}.
\]
If $\bar a_n=C_a(\log(en))^{c_a}$ bounds the envelope in
Assumption~\ref{as:stable-offset}, set
$H_n=C_H\bar a_n\max\{1,R_n,\Tmax\}$, with $C_H$ covering the
constituent ranges, and fit
\begin{equation}
\calG_\ell=\{\Gamma_{H_n}(f\circ c_n):
f\in\Phi_{d+1,d}(L_\Delta,\nu_\Delta,P(L_\Delta,\nu_\Delta),B_\Delta)\}.
\label{eq:fitted-offset-sieve}
\end{equation}
Coordinates specified as zero remain zero; depth-zero projections use
identity channels. If the offset is known to vanish, take
$\calG_\ell=\{0\}$. Otherwise the full dense count satisfies
$P(L_\Delta,\nu_\Delta)\leq CN_*\polylog(n)
\leq Cm\rho_{n,\ell}\polylog(n)$.

\begin{lemma}[Coefficients of the growing-width construction]
\label{lem:bounded-offset-construction}
For a fixed constituent arity $K$, smoothness $q\geq1$, and regularity
bound, the growing-width construction at an admissible resolution $M$ on
$[-a,a]^K$, $a\geq1$, admits coefficient bound
\[
\log(B_{M,a}+1)\leq C(1+\log(M+1))(1+a+\log(M+1)).
\]
Consequently, on boxes of radius $O(\log n)$ and for $M$ polynomial in
$n$, a fixed hierarchy admits depth $C\log(en)$, width
$CN_*^{1/2}\polylog(n)$, coefficient bound
$\exp\{C(\log(en))^2\}$, and dense parameter count $CN_*\polylog(n)$.
\end{lemma}
\begin{proof}
In the construction of
\citet[Theorem~2(a), Lemmas~3--7 and~9--10]{kohlerlanger2020},
selector margins $M^{-2q-2}$ require slopes bounded by
$C(M^{2q+2}+a+1)^2$. Grid biases and Taylor lookup coefficients
are polynomial in $M,a$ and the regularity bound; fixed-degree
arithmetic uses $O(\log M)$ multiplication layers.
The boundary gate has coefficient
$B_{\rm true}=2e^{2aK}\max\{\|f\|_{C^u},1\}$, where $q=u+\alpha$
in our smoothness convention. Hence each primitive coefficient is at
most $e^{Ca}(M+1)^C(a+1)^C$.
Merging affine interfaces multiplies such coefficients and sums at
most $CM^K$ terms per layer. Across $O(1+\log M)$ layers this gives
the displayed logarithmic bound. Fixed compositions, identity padding,
and intermediate clamps preserve it; restoring the envelope adds
only polylogarithmic coefficients. With
$M_v=\lceil N_v^{1/(2K_v)}\rceil\polylog(n)$, width squared times
depth gives the stated dense count.
\end{proof}

To verify \eqref{eq:network-entropy}, let $R$ bound the image of
$c_n$. Parameter perturbations satisfy
\[
\|f_\theta\circ c_n-f_{\theta'}\circ c_n\|_{2,\infty}
\leq H\|\theta-\theta'\|_\infty,
\qquad H=C_d(L+1)(R+1)[(\nu+1)(B+1)]^{L+1}.
\]
Rounding $S$ coefficients on a mesh $\eta/H$ and counting their
supports among $P$ positions gives
\begin{equation}
\log\calN(\eta,\Phi\circ c_n,\|\cdot\|_{2,\infty})
\leq CS\log(eP/S)+S\log(1+CBH/\eta).
\label{eq:sieve-quantization}
\end{equation}
Both $\Gamma_H$ and $\Pi_t$ are nonexpansive, so fixed output clipping
preserves this bound. Under the stated budgets, $\log(BH)$ and
$\log P$ are polylogarithmic.

The bounded-coefficient parameter sets are compact, including the
finite union over sparse supports. Output clipping and
$\Tmin>0$ permit dominated convergence in the empirical loss.
Thus each stage admits a measurable minimizing selection.

\paragraph{Spatial truncation.}
The same tail estimate will be used for either score or the offset.
Let $h\in\{s,s_\ell,\Delta_\ell\}$, let $\mu_t$ be a pooled or group
forward law, and suppose $\sup_{x,t}\|f(x,t)\|_2\leq H$.
If $\|X_0\|_\infty\leq R_0$ for fixed $R_0$, \eqref{eq:score-posterior} implies
$\|h(x,t)\|_2\leq C(1+\|x\|_2)/\sigma_t^2$.
Since $X_t=m_tX_0+\sigma_tZ$, Gaussian tail moments give
\begin{equation}
\begin{aligned}
&\int_{\Tmin}^{\Tmax}\E_{\mu_t}
 \bigl[\|f(X_t,t)-h(X_t,t)\|_2^2
       \one\{\|X_t\|_\infty>R_n\}\bigr]\dd t\\
&\qquad\leq C\log n\,(H^2+\sigma_-^{-4})(1+R_n)^2
       e^{-(R_n-R_0)_+^2/2}.
\end{aligned}
\label{eq:sieve-spatial-tail}
\end{equation}
Indeed, $\|X_t\|_\infty>R_n$ requires $\|Z\|_\infty>R_n-R_0$,
and the same Gaussian bound controls the truncated second moment.
The initial choice of $C_R$ makes this $Cn^{-A}$ whenever
$H^2\leq Cn^{2\kappa_0}\polylog(n)$.

\begin{proposition}[Offset approximation and complexity]
\label{prop:stable-offset}
Assume compact initial support and Assumption~\ref{as:stable-offset}.
Put $m=n_\ell\geq C\log n$ and $\rho_{n,\ell}=\rho_m(\mathbf P_\ell)$.
For any fixed tail exponent $A$, with the truncation radius chosen
as above, the fitted class \eqref{eq:fitted-offset-sieve} satisfies
\[
\begin{aligned}
\mathsf A^{\rm int}_{n,\ell}&\leq C\polylog(n)\rho_{n,\ell}+Cn^{-A},\\
V_{n,\ell}&\leq C\polylog(n)m\rho_{n,\ell},
\end{aligned}
\]
and $\log A_{n,\ell}\leq\polylog(n)$. Its contribution to
\eqref{eq:oracle} is $C\polylog(n/\delta)\rho_{n,\ell}+Cn^{-A}$.
\end{proposition}

\begin{proof}[Proof of Proposition~\ref{prop:stable-offset}]
For a constituent $v$ with pair $(q_v,K_v)$, take
\[
N_v=\left\lceil m^{K_v/(2q_v+K_v)}\right\rceil,\qquad
M_v=\left\lceil N_v^{1/(2K_v)}\right\rceil\polylog(n).
\]
Apply \eqref{eq:constituent-approximation} to the normalized
constituents and compose their approximants, clamping intermediate
ranges. Fixed depth, arity, and Lipschitz constants give, after
restoring the envelope,
\[
\sup_{Q_n}\|g^\circ-\Delta_\ell\|_2^2
\leq C(a^\Delta_{n,\ell})^2\polylog(n)
       \max_v m^{-2q_v/(2q_v+K_v)}
\leq C\polylog(n)\rho_{n,\ell}.
\]
Integration over a window of length $O(\log n)$ preserves this order.
The input and output clamps bound $g^\circ$ by
$Ca^\Delta_{n,\ell}\max\{1,R_n,\Tmax\}$, including depth-zero
projections. Equation~\eqref{eq:sieve-spatial-tail} therefore
bounds the spatial complement by $Cn^{-A}$.

Lemma~\ref{lem:bounded-offset-construction} places the approximant
in \eqref{eq:fitted-offset-sieve}. Since
$N_v/m\leq2m^{-2q_v/(2q_v+K_v)}\leq2\rho_{n,\ell}$,
the padded dense class has at most $Cm\rho_{n,\ell}\polylog(n)$
parameters. Equation~\eqref{eq:network-entropy} gives the claimed
entropy bounds. Finally, $\log K_n^{\max}=O(\log n)$ and
$m^{-1}\leq\rho_{n,\ell}$ absorb the entropy and confidence terms
of \eqref{eq:oracle} into the stated rate.
\end{proof}
\section{Pooled approximation and combination of the stages}
\label{app:mixed-main}

We apply the imported Besov score approximation to the pooled density,
then reuse the oracle, clipping, and entropy bounds from
Appendices~\ref{app:oracle} and~\ref{app:stable-construction}.

\begin{proposition}[Pooled ReLU score rate]
\label{prop:pooled-general}
Under Assumption~\ref{as:pooled-besov}, for $0<\delta\leq1/2$,
the pooled estimator over the fitted sieve \eqref{eq:fitted-pooled-sieve}
satisfies, with probability at least $1-\delta$,
\[
r^{\rm int}_{n,\ell}\leq\bar r_{n,\ell}(\delta)
:=C\bar\kappa\polylog(n/\delta)n^{-2s_0/(2s_0+d)}.
\]
Here $n$ is the independent pooled sample size, $C$ excludes the
displayed overlap factor, and the time window is \eqref{eq:window}.
\end{proposition}

\begin{lemma}[Imported Besov score construction]
\label{lem:proof-oko-input}
Let a density on $[-1,1]^D$ have fixed positive upper and lower
bounds and belong to a fixed restriction Besov ball
$B^s_{p,\theta}$, with $s>D(1/p-1/2)_+$.
Impose the additional hypotheses of
\citet[Assumptions~2.5--2.6]{oko2023}: the schedule is $C^\infty$
with their uniform derivative bounds and
$0<\underline\beta\leq\beta_t\leq\overline\beta$, and the density
has their bounded $C^\infty$ regularity on the collar
\[
[-1,1]^D\setminus[-1+a_N,1-a_N]^D,
\qquad a_N=N^{-(1-\zeta)/D},\quad \zeta\in(0,1).
\]
These regularity bounds are uniform over the resolutions and learning
windows used. For a polynomial lower cutoff and logarithmic upper
cutoff in $N$, their Theorem~3.1 supplies one ReLU network satisfying,
simultaneously for all times in that window,
\[
\int\|f_N(x,t)-\nabla\log p_t(x)\|_2^2p_t(x)\dd x
\leq CN^{-2s/D}\frac{\log N}{\sigma_t^2}.
\]
Its depth, width, sparsity, and weight bound satisfy
\[
L=O(\log^4N),\quad \nu=O(N\log^6N),\quad
S=O(N\log^8N),\quad \log B=O(\log N\log\log N),
\]
and $\sup_x\|f_N(x,t)\|_2=O(\sigma_t^{-1}\sqrt{\log N})$.
\end{lemma}

The collar condition is additional to restriction Besov regularity
(Appendix~\ref{app:notation}). Below, $N\asymp n^\gamma$ for fixed
$\gamma>0$, so \eqref{eq:window} is an admissible window in $N$.
The same holds for target-sized resolutions when $n_\ell\geq n^\epsilon$.

\paragraph{The pooled fitted sieve.}
For $N=\lceil n^{d/(2s_0+d)}\rceil$, fix sufficiently large constants in
\[
\begin{aligned}
L_0&=\lceil C_L(\log(eN))^4\rceil,&
\nu_0&=\lceil C_\nu N(\log(eN))^6\rceil,\\
S_0&=\lceil C_SN(\log(eN))^8\rceil,&
B_0&=\exp\{C_B\log(eN)\log\log(e^2N)\}.
\end{aligned}
\]
With the notation and input clamp of Appendix~\ref{app:fitted-sieves}, set
\begin{equation}
\mathcal F_n=\{\Pi(f\circ c_n):f\in\Phi_{d+1,d}(L_0,\nu_0,S_0,B_0)\},
\qquad \widehat s\in\arg\min_{f\in\mathcal F_n}\calD_n^{\rm int}(f).
\label{eq:fitted-pooled-sieve}
\end{equation}
The sparse budget applies to every candidate; \eqref{eq:sieve-quantization}
counts its support among the $O(N^2\polylog(N))$ available positions.

\begin{proof}[Proof of Proposition~\ref{prop:pooled-general}]
Lemma~\ref{lem:proof-oko-input} and
Lemma~\ref{lem:app-envelope} give
\[
\int_{\Tmin}^{\Tmax}\!\int\|f_N-s_t\|_2^2p_t\dd x\dd t
\leq CN^{-2s_0/d}\log N\,\Lambda_2
\leq C\polylog(n)N^{-2s_0/d}.
\]
Input clamping changes nothing on $Q_n$ and gives global output bound
$C\sigma_-^{-1}\sqrt{\log n}$. Equation~\eqref{eq:sieve-spatial-tail}
therefore bounds the added spatial error by $Cn^{-A}$.
The pooled version of \eqref{eq:app-projection-risk} adds only a
constant factor and $C\tau_n^{\rm int}$.

Equation~\eqref{eq:sieve-quantization} gives
$V_{n,0}\leq CN\polylog(n)$ and $\log A_{n,0}\leq\polylog(n)$.
Applying the pooled version of Lemma~\ref{lem:app-generic-oracle},
then balancing approximation and estimation, yields with probability
at least $1-\delta$,
\[
\begin{aligned}
\int_{\Tmin}^{\Tmax}\!\int\|\widehat s-s_t\|_2^2p_t\dd x\dd t
&\leq C\polylog(n/\delta)\{N^{-2s_0/d}+N/n\}
       +Cn^{-A}+C\tau_n^{\rm int}\\
&\leq C\polylog(n/\delta)n^{-2s_0/(2s_0+d)}.
\end{aligned}
\]
Here $N=\lceil n^{d/(2s_0+d)}\rceil$, $A>1$ is chosen in the
construction, and \eqref{eq:app-tail-def} absorbs the clipping remainder.
Finally, positivity of the OU transition kernel propagates overlap:
$p_{t,\ell}\leq\bar\kappa p_t$. Multiplying the squared-error
integrand by this inequality proves the claimed group-weighted bound.
\end{proof}

\begin{proof}[Proof of Theorem~\ref{thm:mixed-main}]
On the pooled-risk event of Proposition~\ref{prop:pooled-general},
condition on $\widehat s$. Proposition~\ref{prop:stable-offset}
bounds the offset contribution in Theorem~\ref{thm:oracle}; substituting
both propositions gives the claimed rate.
Independence leaves the target sample i.i.d. after conditioning, so
the pooled and conditional target failure probabilities sum to at most
$2\delta$. If the offset is known to vanish and $\calG_\ell=\{0\}$,
the fitted pooled score is already clipped, hence
$\widehat s_\ell=\widehat s$ and the pooled bound alone applies.
\end{proof}

\paragraph{Overlap and the comparison sample.}
Mixture membership gives $p_{0,\ell}/p_0\leq1/\pi_\ell$;
a prevalence-uniform $\bar\kappa$ imposes stronger overlap for rare
groups and implies $\Pp(Z=\ell\mid X=x)\leq\bar\kappa\pi_\ell$.
If instead $\bar\kappa_n\asymp n^\gamma$ while all other constants
remain uniform, the pooled exponent becomes $\alpha_0-\gamma$ and
the first transfer condition becomes
$\alpha_0-\gamma>\rho\alpha_{\rm sep}$.

Separate in Section~\ref{sec:combined} uses the $n_\ell$ adaptation
observations. If pooled target labels are available, an
information-matched target-only fit could use
$n_\ell+N_\ell^{\rm pool}$ observations, where
$N_\ell^{\rm pool}=\sum_{i=1}^n\one\{Z_i=\ell\}$, with benchmark
$\widetilde O((n_\ell+N_\ell^{\rm pool})^{-\alpha_{\rm sep}})$
on an admissible window. At fixed prevalence,
$N_\ell^{\rm pool}\asymp n\pi_\ell$ with high probability.
Our comparison concerns adaptation to an available pooled fit;
Appendix~\ref{app:experimental-protocols} records the data allocations.

\begin{example}[Gains from learning the shared score with pooled data]
\label{ex:pooling-rates}
Suppose $s_0=s_\ell=s$, set $a=2s/(2s+d)$, and suppose
$b=\alpha_\Delta>a$. The bounds are
\[
\text{Separate: }\widetilde O(n_\ell^{-a}),
\qquad
\text{Two-stage: }\widetilde O(n^{-a}+n_\ell^{-b}).
\]
Pooling learns the shared component from $n$ observations and uses
the target sample for the simpler offset. For
$n_\ell\asymp n^\rho$ with $0<\rho<1$, both two-stage terms decay
faster than the separate-fitting bound.
\end{example}
\section{Small-time regularity of the offset}
\label{app:offset}

We prove uniform boundedness first, then combine interior smoothness with
the boundary cancellation from Appendix~\ref{app:boundary-details}.

\subsection{Uniform boundedness}
\begin{proof}[Proof of Proposition~\ref{prop:ratio-cancellation}]
Gaussian density ratios give $G=\E_{\bar\pi_{z,v}}h(U)$ and
\eqref{eq:general-ratio}. The posterior identity
\eqref{eq:score-posterior} then gives
\[
\Delta_{t,\ell}(x)=\frac{m_t}{\sigma_t^2}
\frac{\Cov_{\bar\pi_{z,v}}(h(U),U)}{G(z,v)}.
\]
Let $\nu_{z,v}$ be $N(z,vI_D)$ conditioned on the convex cube.
The Brascamp--Lieb inequality~\citep{brascamplieb1976}, applied by convex
penalization outside the cube, gives
$\Var_\nu(g)\leq v\operatorname{Lip}(g)^2$.
Bounded reweighting by $r$ and the variational formula for variance yield
\[
\Var_{\bar\pi_{z,v}}(g)\leq(C_r/c_r)v\operatorname{Lip}(g)^2.
\]
Apply this to $h$ and each coordinate of $U$. Cauchy--Schwarz and
$G\geq h_{\min}$ give, for $t\leq1$,
\[
|\Delta_{t,\ell}^{(j)}(x)|
\leq\frac{Cm_tv_t\operatorname{Lip}(h)}{\sigma_t^2h_{\min}}
=\frac{C\operatorname{Lip}(h)}{m_th_{\min}}\leq C.
\]
For $t\geq1$, bounded support bounds the covariance directly, and
$m_t/\sigma_t^2\leq m_1/\sigma_1^2$.
\end{proof}

\subsection{Joint regularity under a smooth collar-flat ratio}
Work on $K=[-1,1]^k$ with pooled density $\bar r$ and group density
$h\bar r$, where $h=h_\ell$. Put $r_{\Delta,\ell}=r_{\mathrm{off},\ell}-1$.
For $z=x/m_t$ and $v_t=\sigma_t^2/m_t^2$,
\begin{equation}
\Delta_{t,\ell}^{(j)}(x)=m_t^{-1}\partial_{z_j}\log G_\ell(z,v_t),
\qquad
G_\ell(z,v)=\frac{P_v(h\bar r)(z)}{P_v\bar r(z)},
\label{eq:offset-representation}
\end{equation}
where $P_vf=f*\varphi_{vI_k}$ and the densities are extended by zero from $K$. Write
$D=P_v\bar r$; the associated posterior is
\begin{equation}
\bar\pi_{z,v}(\dd u)=D(z,v)^{-1}\bar r(u)\varphi_{vI_k}(z-u)
\one_K(u)\dd u.
\label{eq:app-posterior}
\end{equation}

\paragraph{Heat-semigroup estimates.}
For $f\in C^{m,\alpha}(\R^k)$, $p=|\eta|_1+2b\leq m$, and $0<w\leq1$,
\begin{equation}
\begin{aligned}
\partial_v^b\partial_z^\eta P_vf
&=2^{-b}P_v\Delta^b\partial^\eta f,\\
\|\partial_v^b\partial_z^\eta(P_{v+w}-P_v)f\|_\infty
&\leq Cw^{\min\{2,m+\alpha-p\}/2}.
\end{aligned}
\label{eq:app-heat-increments}
\end{equation}
Indeed, $P_{v+w}=P_vP_w$ and the heat identity give the first equality.
Apply the remaining H\"older regularity to $\sqrt wZ$; when it exceeds
one, subtract the linear Taylor term, whose Gaussian expectation is zero.
Supremum contraction gives the increment bound and preserves spatial
H\"older exponents $\min\{1,m+\alpha-p\}$.

\begin{lemma}[Weighted derivatives of a collar-flat tilted ratio]
\label{lem:regularity-weighted-ratio}
Suppose $0<c_r\leq r\leq C_r$ on $K=[-1,1]^k$, and
$0<h_{\min}\leq h\leq h_{\max}$ with $h=c_\partial$ on a collar of
width $a_0$. For $\kappa\geq0$, define
\[
B(\theta,\kappa)
=\log\frac{\int_Kh(u)r(u)e^{\theta^\top u-\kappa\|u\|_2^2}\dd u}
                {\int_Kr(u)e^{\theta^\top u-\kappa\|u\|_2^2}\dd u}.
\]
For any fixed $M\geq0$ and positive derivative order $|\eta|_1+b>0$,
\[
\sup_{\theta\in\R^k}(1+\|\theta\|_1)^M
|\partial_\theta^\eta\partial_\kappa^bB(\theta,\kappa)|
\leq C_{M,\eta,b}(1+\kappa)^M.
\]
\end{lemma}

\begin{proof}
Put $\rho=\|\theta\|_1$ and let $\nu_{\theta,\kappa}$ be the posterior
proportional to $r(u)e^{\theta^\top u-\kappa\|u\|_2^2}$.
On $K^\circ=[-1+a_0,1-a_0]^k$, $\theta^\top u\leq(1-a_0)\rho$.
An inward box of side $[C(1+\rho+\kappa)]^{-1}$ at a maximizing corner
of $K$ has $\theta^\top u\geq\rho-C$ and $\|u\|_2^2\leq k$. Hence
\[
\nu_{\theta,\kappa}(K^\circ)
\leq C(1+\rho+\kappa)^k e^{-a_0\rho+k\kappa}.
\]
Parameter derivatives insert bounded products of $u$ and $-\|u\|_2^2$.
The quotient recursion used in the proof of
Proposition~\ref{prop:collar-cancellation}
therefore bounds every positive-order derivative of $B$ by
$C\nu_{\theta,\kappa}(K^\circ)$: its nonconstant numerator is supported
in $K^\circ$, while relative denominator derivatives are bounded.
For $\rho\geq1+(2k/a_0)\kappa$, the exponential absorbs every power of
$1+\rho$. Otherwise $(1+\rho)^M\leq C_M(1+\kappa)^M$, and bounded
sufficient statistics bound all fixed parameter derivatives directly.
\end{proof}

\paragraph{Bounds away from zero.}
The preceding lemma requires no spatial smoothness of $r$.
Let $\lambda_t=m_t/\sigma_t^2$, $\kappa_t=m_t^2/(2\sigma_t^2)$, and
$A_r(\theta,\kappa)=\log\int_Kr(u)e^{\theta^\top u-\kappa\|u\|_2^2}\dd u$.
Completing the square gives
\begin{equation}
\nabla\log r_t(x)=-x/\sigma_t^2
+\lambda_t\nabla_\theta A_r(\lambda_tx,\kappa_t).
\label{eq:app-exp-family-score}
\end{equation}
Thus, with $B=A_{hr}-A_r$,
\begin{equation}
\Delta^{(j)}(x,t)=\lambda_t\partial_{\theta_j}B(\lambda_tx,\kappa_t).
\label{eq:app-tilted-offset}
\end{equation}
Since $\lambda_t'/\lambda_t=-\beta_t-2\beta_tm_t^2/\sigma_t^2$
(and similarly for $\kappa_t$), differentiation introduces only powers
of $\theta_t=\lambda_tx$, $\kappa_t$, and $\sigma_t^{-1}$, besides
bounded schedule derivatives and positive-order derivatives of $B$.
Lemma~\ref{lem:regularity-weighted-ratio} absorbs the $\theta_t$ powers
into powers of $1+\kappa_t\leq C\sigma_t^{-2}$.
For fixed integer $q\geq0$ and $\tau>0$, if the schedule derivatives
through order $q$ are bounded, then
\begin{equation}
\sup_{x\in\R^k,\,t\geq\tau}
|\partial_t^b\partial_x^\gamma\Delta^{(j)}(x,t)|
\leq C_q(1+\sigma_\tau^{-1})^{C_q},\qquad |\gamma|_1+b\leq q.
\label{eq:app-away-time}
\end{equation}

\begin{lemma}[Parabolic derivative envelope]
\label{lem:app-offset-envelope}
Under the conditions of Proposition~\ref{prop:offset}, for $R\geq1$,
$j\in[k]$, and all integer pairs $(\gamma,b)$ with
$|\gamma|_1+2b\leq m_{\mathrm{off},\ell}-1$,
\begin{equation}
\sup_{\substack{\norm{x}_\infty\leq R\\t\in[\Tmin,\Tmax]}}
|\partial_t^b\partial_x^\gamma\Delta_{t,\ell}^{(j)}(x)|
\leq C(1+R)^b\polylog(n).
\label{eq:app-parabolic-envelope}
\end{equation}
At parabolic order $m_{\mathrm{off},\ell}-1$, the same bound holds for the
$\alpha_{\mathrm{off},\ell}$ spatial and $\alpha_{\mathrm{off},\ell}/2$
temporal H\"older seminorms. The constants are uniform over
$\Tmin\geq n^{-\kappa_0}$.
\end{lemma}

\begin{proof}
Put $m=m_{\mathrm{off},\ell}$, $\alpha=\alpha_{\mathrm{off},\ell}$, and
$K^\circ=[-1+a_0,1-a_0]^k$. The zero extension of
$(h-c_\partial)\bar r$ is $C^{m,\alpha}$, since it vanishes on a full
boundary collar. With $N=P_v[(h-c_\partial)\bar r]$,
\begin{equation}
G_\ell=c_\partial+N/D,\qquad G_\ell\geq h_{\min}>0.
\label{eq:app-G-ND}
\end{equation}
Choose $t_\times=(c_\times/\log n)\wedge1$ so that $v_{t_\times}$ is
smaller than a fixed multiple of $a_0^2$.
For $t\leq t_\times$, we transfer derivatives to the smooth densities
in the interior and use boundary cancellation elsewhere.

\emph{Interior.}
On $\operatorname{dist}(z,K^\circ)\leq a_0/3$, a fixed ball about $z$
lies in $K$, so $D\geq c>0$ for small $v$.
Extend $\bar r$ to a compactly supported $C^{m,\alpha}$ function
\citep[Chapter~VI]{stein1970}. Its convolution differs from $D$ by
an error $\mathcal E$ whose integrand is supported outside $K$;
Lemma~\ref{lem:boundary-kernel} gives
$|\partial_z^\eta\partial_v^a\mathcal E|\leq Cv^{-M}e^{-ca_0^2/v}$.
The heat bounds \eqref{eq:app-heat-increments} apply to $N$ and the
extended denominator. Quotient differentiation and the lower bounds
for $D$ and $G_\ell$ therefore control derivatives of $\log G_\ell$
with parabolic budget $|\eta|_1+2a\leq m$.

\emph{Boundary and exterior.}
On $\operatorname{dist}(z,K^\circ)\geq a_0/4$,
Lemma~\ref{lem:boundary-cube-mass} gives
\begin{equation}
\bar\pi_{z,v}(K^\circ)
\leq C\exp\{-c(a_0^2+a_0\operatorname{dist}(z,K))/v\}.
\label{eq:app-posterior-tail}
\end{equation}
Proposition~\ref{prop:collar-cancellation}, specifically
\eqref{eq:boundary-weighted-cancellation}, then controls every fixed
mixed derivative of $\log G_\ell$, including arbitrary fixed polynomial
weights in $z$ and $v^{-1}$.

\emph{Space--time derivatives.}
The chain rules are
\begin{equation}
\partial_tz=\beta_tz,\qquad \partial_tv=2\beta_tm_t^{-2},\qquad
\partial_vP_vf=\tfrac12\Delta_zP_vf.
\label{eq:app-parabolic-chain}
\end{equation}
Applied to \eqref{eq:offset-representation}, they express each derivative
of the offset as a finite sum
\[
c_{\eta,a}(z,t)\partial_z^\eta\partial_v^a\log G_\ell(z,v_t),
\qquad |\eta|_1+2a\leq1+|\gamma|_1+2b.
\]
The coefficients have polynomial degree at most $b$ in $z$ and bounded
remaining factors for $t\leq t_\times$.
Interior boundedness of $z$ and exterior localization therefore give
uniform bounds whenever $1+|\gamma|_1+2b\leq m$.

\emph{H\"older increments.}
At primitive parabolic order $p\leq m$,
\eqref{eq:app-heat-increments} gives variance exponent
$\min\{2,m+\alpha-p\}/2$ and spatial exponent
$\min\{1,m+\alpha-p\}$ in the interior.
Products and reciprocals preserve these bounds because the denominators
are bounded below. Exterior localization controls additional derivatives
of every fixed order. At offset order $m-1$, the gradient uses one more
primitive order, leaving spatial exponent $\alpha$ and variance exponent
$\alpha/2$. The smooth maps $z=x/m_t$ and $v=v_t$ convert these to the
stated space--time exponents; polynomial coefficients are controlled as
above. The overlapping zones allow increments across their interface.

\emph{Larger times.}
For $t\geq t_\times$, $\sigma_t^{-1}\leq C\sqrt{\log n}$, so
\eqref{eq:app-away-time} bounds the required derivatives by
$C\polylog(n)$ uniformly in $x$.
One additional derivative gives the H\"older bounds.
Splitting increments at $t_\times$ completes the full-window estimates.
\end{proof}

\paragraph{Extension from rectangles.}
\label{app:offset-extension}
A $C^{u,\vartheta}$ function on a rectangle whose side lengths are
bounded below has a global extension with
$C^{u,\vartheta}$ norm at most a fixed multiple of its original norm
\citep[Chapter~VI]{stein1970}. One may match derivatives through order
$u$ by finite linear combinations of reflections at each face, then
apply fixed-width smooth cutoffs. Tensoring these one-dimensional
operators gives constants independent of the rectangle's location
and side lengths. This applies to $Q_n$ since $R_n\geq1$ and
$T_n\geq c_0$.

\begin{proof}[Proof of Proposition~\ref{prop:offset}]
Boundedness follows from Proposition~\ref{prop:ratio-cancellation} and
integer derivative bounds from Lemma~\ref{lem:app-offset-envelope}.
For isotropic smoothness, write $s=r_{\Delta,\ell}/2=u+\vartheta$,
$\vartheta\in(0,1]$. An ordinary derivative of order $u$ uses primitive
parabolic order at most $1+2u$, leaving
$r_{\mathrm{off},\ell}-(1+2u)=2\vartheta$.
The heat and quotient increment bounds in the preceding proof give
spatial and temporal exponents at least $\vartheta$.
Thus the $C^{u,\vartheta}$ norm on $Q_n$ is $\polylog(n)$.
The rectangle extension above preserves this bound; since $s>1$, it
also controls the global Lipschitz constant. Normalizing by the envelope
gives the required depth-one constituent.
\end{proof}
\section{A rough density with a stable smooth offset}
\label{app:rough-density-certificate}

Separating a pooled singularity from the group-specific change gives
an explicit instance of Theorem~\ref{thm:mixed-main}.
Appendix~\ref{app:rough-density-protocol} reports its simulation.
Let $K=[-1,1]^2$, $\beta_t=1$, and define
\[
B(s)=\begin{cases}
\exp\{1-(1-s^2)^{-1}\},&|s|<1,\\
0,&|s|\geq1.
\end{cases}
\]
Set $a=-1/2$, $c=9/20$, $w=1/5$, $\alpha=1/2$, and
\[
\begin{aligned}
C_\alpha&=\frac{\int_{-1}^1B(s)|s|^\alpha\dd s}
{\int_{-1}^1B(s)\dd s},\qquad b(u)=B(2u),\\
r(u)&=B((u-a)/w)\bigl(|(u-a)/w|^\alpha-C_\alpha\bigr),\\
A(u)&=B((u-c)/w)\sin(\pi(u-c)/w),\\
\psi(x)&=A(x_1)b(x_2).
\end{aligned}
\]
For $0<\pi_\ell\leq1/2$, take
\begin{equation}
\begin{aligned}
p_0(x)&=\tfrac14\{1+\tfrac12r(x_1)b(x_2)\},\\
p_{0,\ell}(x)&=p_0(x)\{1+\tfrac12\psi(x)\},\\
p_{0,-\ell}(x)&=p_0(x)\left\{1-
\frac{\pi_\ell}{2(1-\pi_\ell)}\psi(x)\right\},\quad x\in K,
\end{aligned}
\label{eq:rough-density-witness}
\end{equation}
and zero outside $K$.

\begin{proposition}[Separated cusp and smooth tilt]
\label{prop:rough-density-witness}
The densities in \eqref{eq:rough-density-witness} satisfy
Assumption~\ref{as:pooled-besov} with $d=2$,
$s_0=1/2$, $p_0^\ast=\theta_0^\ast=\infty$, and
$\bar\kappa=3/2$. The pooled density is not locally in
$B^s_{\infty,\infty}$ at the cusp for any $s>1/2$.
Both group offsets satisfy Assumption~\ref{as:stable-offset} with
depth one, $\mathcal P_j=\{(2,3)\}$, and
$a^\Delta_{n,g}\leq C(1+\log n)$, uniformly in prevalence and the
admissible lower cutoff.
\end{proposition}

\begin{proof}
\emph{Densities and roughness.}
The rough and tilt supports are
$[-7/10,-3/10]\times[-1/2,1/2]$ and
$S=[1/4,13/20]\times[-1/2,1/2]$, separated by $11/20$.
Thus $r(x_1)b(x_2)\psi(x)=0$, and for either group
\[
p_{0,g}(x)=\tfrac14\{1+\tfrac12r(x_1)b(x_2)+c_g\psi(x)\},
\qquad |c_g|\leq1/2.
\]
Since $\int r=\int A=0$, these are normalized densities with the stated
pooled mixture. Disjoint supports and $|r|,|A|,|b|\leq1$ give
$1/8\leq p_{0,g},p_0\leq3/8$ and $1/2\leq p_{0,g}/p_0\leq3/2$.
All densities equal $1/4$ on $K\setminus[-3/4,3/4]^2$, supplying
the boundary collar of Lemma~\ref{lem:proof-oko-input}.
A smooth extension of the constant part plus the interior perturbation
has fixed $1/2$-H\"older norm. At $(a,0)$, its second spatial difference
is a nonzero multiple of $|h|^{1/2}+O(h^2)$, establishing the claimed
Besov regularity and roughness.

\emph{Small-time offset.}
Support separation gives, with $z=e^tx$ and $v=e^{2t}-1$,
\[
D=P_v[\one_K(1+\tfrac12rb)],\quad N=P_v\psi,\quad
G_g=1+c_gN/D\geq1/2,\quad
\Delta_g=e^t\nabla_z\log G_g.
\]
Within distance $1/16$ of $S$, $D\geq c>0$ for small $v$.
Writing $D=1+\tfrac12P_v(rb)-P_v\one_{K^c}$,
Lemma~\ref{lem:boundary-kernel} bounds its positive-order derivatives
by $Cv^{-M}e^{-c/v}$, since both perturbations are separated from this
neighborhood. The heat estimates \eqref{eq:app-heat-increments} apply
to the smooth numerator. Quotient differentiation therefore bounds all
fixed mixed derivatives of $\log G_g$ there.
Elsewhere, apply the support-localized form of
Proposition~\ref{prop:collar-cancellation} with $r=p_0$ and
$h=1+c_g\psi$. Estimate \eqref{eq:boundary-weighted-cancellation}
absorbs the polynomial $z$ factors from \eqref{eq:app-parabolic-chain}.
Consequently all space--time derivatives of $\Delta_g$ through total
order two are bounded at small times.

\emph{Remaining times and extension.}
For any fixed $\tau>0$, apply \eqref{eq:app-away-time} with
$r=p_0$, $h=1+c_g\psi$, and $q=2$; this estimate requires no
smoothness of the pooled density.
Together with the small-time bounds and the rectangle extension in
Appendix~\ref{app:offset-extension}, this yields global bounded,
Lipschitz $C^{1,1}$ constituents after normalization by
$\max\{1,C(1+R_n)^2\}\leq C(1+\log n)$.
All constants are uniform in $|c_g|\leq1/2$.
\end{proof}

Taking $m=n_\ell\asymp n^{3/4}$ and $\pi_\ell=m/n$ gives polynomial
upper-bound terms $n^{-1/3}+m^{-4/7}$, versus $m^{-1/3}$ for separate
Besov estimation on a common admissible interval, up to logarithms.
The gain here comes from a smoother offset, not from reducing the spatial dimension.
\section{Proofs for total-variation recovery}
\label{app:tv}
We condition on the fitted score and use the distance conventions of
Appendix~\ref{app:notation}. The decomposition below follows
\citet[Appendix~D.1, Eq.~(78)]{oko2023}; we verify its endpoint and
path-space bounds under our assumptions.

\begin{proposition}[Three-regime TV decomposition, adapted from Oko et al.]
\label{prop:tv}

Assume the exact and learned reverse drifts are measurable in time,
globally Lipschitz in space, and of linear growth, uniformly on the fixed
learning interval. Then
\begin{equation}
\begin{aligned}
\TV(P_{0,\ell},\widehat P_\ell)
\leq{}&\TV(P_{0,\ell},P_{\Tmin,\ell})+\TV(P_{\Tmax,\ell},\gamma_d)\\
&+\left\{\frac{\overline\beta}{2}
\calR^{\rm int}_\ell(\widehat s_\ell)\right\}^{1/2}.
\end{aligned}
\label{eq:app-tv-three-regimes}
\end{equation}
The first term is at most $C\sqrt{\Tmin}$ if $p_{0,\ell}$ has fixed
compact support and uniformly bounded zero-extension $BV$ norm. The
second is at most $Ce^{-\underline\beta\Tmax}$ if
$\KL(P_{0,\ell}\|\gamma_d)$ is uniformly bounded.
\end{proposition}

The $BV$ condition allows jumps at the support boundary.

\begin{lemma}[Early-time OU approximation in $L^1$]
\label{lem:app-bv}
Let $P_0$ have density $p$ supported on $[-1,1]^d$, and suppose its zero extension belongs
to $BV(\R^d)$.  Then, for $0<t\leq t_0$,
\begin{equation}
\TV(P_0,P_t)\leq C\sqrt t,
\label{eq:app-bv}
\end{equation}
where $C$ depends on the $BV$ norm, $d$, $t_0$, and the schedule bounds.
\end{lemma}

\begin{proof}
We put $m=m_t$, $v=\sigma_t^2/m_t^2$, and let $Q_v$ be the law of
$X_0+\sqrt vZ$. Since $P_t=(D_m)_\#Q_v$, contraction and the triangle
inequality give
\begin{equation}
\TV(P_0,P_t)
\leq\TV(P_0,(D_m)_\#P_0)+\TV(P_0,Q_v),
\label{eq:app-bv-split}
\end{equation}
with $D_m(x)=mx$. Integrating the standard $BV$ translation bound
\citep[Chapter~3]{ambrosio2000} against $\sqrt vZ$ gives
\begin{equation}
\norm{p-p*\varphi_{vI_d}}_1
\leq\sqrt v\,\E\norm Z_2\,|Dp|(\R^d).
\label{eq:app-bv-smoothing}
\end{equation}
For dilation, set $p_r(x)=r^{-d}p(x/r)$. Differentiating a
smooth approximation gives
$\|\partial_rp_r\|_1\leq r^{-1}(d\|p\|_1+\int\|y\|_2\,\dd|Dp|(y))$.
Integrating from $m$ to one and passing through $BV$ mollification gives
\begin{equation}
\norm{p-p_m}_1\leq C(1-m)\norm p_{BV}.
\label{eq:app-bv-dilation}
\end{equation}
Finally $m_t\geq e^{-\overline\beta t_0}$, $1-m_t\leq\overline\beta t$,
and $v\leq Ct$. Substitution yields $C(\sqrt t+t)\leq C\sqrt t$.
\end{proof}

We next control the terminal distribution through Gaussian entropy decay. 

\begin{lemma}[Terminal Gaussian approximation]
\label{lem:app-terminal}
If $\KL(P_0\|\gamma_d)<\infty$, then
\begin{equation}
\TV(P_t,\gamma_d)
\leq e^{-\int_0^t\beta_u\dd u}
\sqrt{\frac12\KL(P_0\|\gamma_d)}
\leq C e^{-\underline\beta t}.
\label{eq:app-terminal}
\end{equation}
\end{lemma}

\begin{proof}
At effective time $\mathcal B_t=\int_0^t\beta_u\dd u$, the Gaussian
log-Sobolev inequality~\citep{gross1975} gives
$\KL(P_t\|\gamma_d)\leq e^{-2\mathcal B_t}\KL(P_0\|\gamma_d)$.
Apply Pinsker's inequality~\citep[Chapter~2]{tsybakov2009} and
$\mathcal B_t\geq\underline\beta t$.
\end{proof}

\paragraph{Localization for the fitted networks.}
For $X_0\in[-1,1]^d$, fixed $n$, and $\Tmin>0$,
Proposition~\ref{prop:score-boundary} and
$\|\Cov(X_0\mid X_t=x)\|_{\rm op}\leq d$ give
\[
 \|s_{t,\ell}(x)\|_2\leq
 \frac{\|x\|_2+\sqrt d}{\sigma_t^2},\qquad
 \|\nabla s_{t,\ell}(x)\|_{\rm op}
 \leq\sigma_t^{-2}+d\sigma_t^{-4}.
\]
The fitted networks and their fixed clamps and projections are globally
Lipschitz in $x$. Thus both reverse drifts satisfy the stated conditions
and admit nonexplosive, pathwise-unique solutions. The constants may
depend on $n$.

\begin{lemma}[Middle-time Girsanov bound]
\textnormal{(Adapted from \citealp[Proposition~D.1]{oko2023}.)}
\label{lem:app-girsanov}
Let $\widetilde P_{\Tmin,\ell}$ be the terminal law of the learned reverse SDE
\eqref{eq:reverse-sde} initialized from $P_{\Tmax,\ell}$.
Under the localization assumptions of Proposition~\ref{prop:tv},
\begin{equation}
\TV(P_{\Tmin,\ell},\widetilde P_{\Tmin,\ell})
\leq\left\{\frac{\overline\beta}{2}
\calR^{\rm int}_\ell(\widehat s_\ell)\right\}^{1/2}.
\label{eq:app-girsanov}
\end{equation}
\end{lemma}

\begin{proof}
Initialize both reverse processes from $P_{\Tmax,\ell}$ and put
$U=\Tmax-\Tmin$. By \eqref{eq:reverse-sde}, their normalized drift
difference is
\[
 h_u(y)=\sqrt{2\beta_{\Tmax-u}}
       \{s_{\Tmax-u,\ell}(y)-\widehat s_\ell(y,\Tmax-u)\}.
\]
Under the exact reverse law $\mathbb Q$, the marginal at reverse
time $u$ is $P_{\Tmax-u,\ell}$, so
\begin{align*}
 \frac12\E_{\mathbb Q}\int_0^U\|h_u(Y_u)\|_2^2\dd u
 &=\int_{\Tmin}^{\Tmax}\beta_t
   \E_{p_{t,\ell}}\norm{s_{t,\ell}(X_t)-\widehat s_\ell(X_t,t)}_2^2\dd t\\
 &\leq\overline\beta\calR^{\rm int}_\ell(\widehat s_\ell).
\end{align*}
When this energy is infinite, \eqref{eq:app-girsanov} is immediate.
For finite energy, we stop the accumulated normalized drift energy at
\[
 \tau_M=\inf\left\{u:\int_0^u\|h_v(Y_v)\|_2^2\dd v\geq M\right\}\wedge U.
\]
Under the exact law $\mathbb Q$, the exponential with integrand
$-h\one_{[0,\tau_M]}$ is a martingale, since its quadratic variation is
bounded by $M$. The resulting law $\mathbb Q_M$ uses the learned drift
until $\tau_M$ and the exact drift afterwards. Its initial law remains
$P_{\Tmax,\ell}$, and
\[
 \KL(\mathbb Q\|\mathbb Q_M)
 =\frac12\E_{\mathbb Q}\int_0^{\tau_M}\|h_u(Y_u)\|_2^2\dd u
 \leq\overline\beta\calR^{\rm int}_\ell(\widehat s_\ell).
\]
We couple this switched process to the fully learned process with the
same initial variable and Brownian motion. Pathwise uniqueness makes
the paths identical until the threshold is reached. Continuity,
nonexplosion, and linear growth of $h$ give finite accumulated energy
on $[0,U]$ almost surely under the learned law. Thus its probability of
reaching threshold $M$ tends to zero. The coupling implies TV convergence
of the switched endpoint law to $\widetilde P_{\Tmin,\ell}$.
Endpoint data processing, Pinsker's inequality, and this limit prove
\eqref{eq:app-girsanov}.
\end{proof}

\begin{proof}[Proof of Proposition~\ref{prop:tv}]
The learned reverse kernel $\widehat K$ sends $P_{\Tmax,\ell}$ to
$\widetilde P_{\Tmin,\ell}$ and $\gamma_d$ to $\widehat P_\ell$;
TV contraction~\citep[Chapter~2]{tsybakov2009} bounds the difference by
$\TV(P_{\Tmax,\ell},\gamma_d)$.
The triangle inequality through $P_{\Tmin,\ell}$ and
$\widetilde P_{\Tmin,\ell}$ then gives the stated decomposition.
Lemmas~\ref{lem:app-bv}, \ref{lem:app-terminal}, and
\ref{lem:app-girsanov} give the three bounds.
\end{proof}
\section{Proof of the balanced TV rate}
\label{app:tv-rate}

Write
\begin{equation}
S^{\rm mix}_{n,\ell}:=n^{-2s_0/(2s_0+d)}+n_\ell^{-\alpha_\Delta}.
\label{eq:working-score-rate}
\end{equation}
Omit the offset term only when the offset is identically zero and
$\mathcal G_\ell=\{0\}$.
The corresponding TV rate is
\begin{equation}
\varepsilon^{\rm TV}_{n,\ell}:=
\polylog(n)^{1/2}(S^{\rm mix}_{n,\ell})^{1/2}.
\label{eq:tv-target-rate}
\end{equation}
The admissible cutoffs in Corollary~\ref{cor:tv-rate} satisfy
\begin{equation}
\sqrt{\Tmin}+e^{-\underline\beta\Tmax}
\lesssim (S^{\rm mix}_{n,\ell})^{1/2}.
\label{eq:balanced-cutoffs}
\end{equation}
For $n_\ell\asymp n^\rho$, put
$\alpha_0=2s_0/(2s_0+d)$ and
$\eta_\ell=\min\{\alpha_0,\rho\alpha_\Delta\}$ (or $\eta_\ell=\alpha_0$
when the offset term is omitted). If $\kappa_0\geq\eta_\ell$, an admissible choice is
$\Tmin=c_Tn^{-\kappa}$, $\Tmax=c_S\log n$, with
\[
 \eta_\ell\leq\kappa\leq\kappa_0,\qquad
 c_S\underline\beta\geq\eta_\ell/2,
\]
and constants satisfying \eqref{eq:window}. If $\kappa_0\geq\alpha_0$,
no relation between sample sizes is needed: replace $\eta_\ell$ by $\alpha_0$.
This suffices
because $S^{\rm mix}_{n,\ell}\geq n^{-\alpha_0}$.

\begin{proof}[Proof of Corollary~\ref{cor:tv-rate}]
Compact support, $p_0\leq C_0$, and overlap give
$p_{0,\ell}\leq M:=\bar\kappa C_0$ on $[-1,1]^d$. Therefore
\[
\KL(P_{0,\ell}\|\gamma_d)
\leq \log M+\frac d2\log(2\pi)+\frac d2.
\]
Together with the stated $BV$ and localization conditions, this verifies
the hypotheses of Proposition~\ref{prop:tv}. Apply
Theorem~\ref{thm:mixed-main} with $\delta=n^{-b}/2$ and use
\eqref{eq:balanced-cutoffs}. With probability at least $1-n^{-b}$,
Proposition~\ref{prop:tv} gives
$\TV(P_{0,\ell},\widehat P_\ell)\leq C\varepsilon^{\rm TV}_{n,\ell}$.
Splitting the square root in \eqref{eq:working-score-rate} proves
\eqref{eq:tv-rate}.
\end{proof}
\section{Additional regularity settings}
\label{app:alternatives}

\subsection{Hierarchical composition in the joint space--time input}
\label{sec:weak-hierarchy}

Here both stages use composition classes whose arities count time as an input.

\begin{assumption}[Joint space--time hierarchical composition for both stages]
\label{as:weak-hierarchy}
On $Q_n$, with the fixed spatial constant $C_R$ chosen sufficiently large
for the tail bound \eqref{eq:sieve-spatial-tail}, there are global extensions
$\widetilde s_n$ and $\widetilde\Delta_{n,\ell}$ such that
\[
 s=a_{n,0}\widetilde s_n,\qquad
 \Delta_\ell=a^\Delta_{n,\ell}\widetilde\Delta_{n,\ell}
 \quad\text{on }Q_n.
\]
The extensions belong, respectively, to
$\mathcal H_{d+1,d}^{\rm coord}
(\mathbf M_0,\mathbf P_0^{\rm st};C)$ and
$\mathcal H_{d+1,d}^{\rm coord}
(\mathbf M_{\Delta,\ell},\mathbf P_{\Delta,\ell}^{\rm st};C)$,
where all structural parameters and $C$ are independent of $n$.
The nonnegative envelopes satisfy
$a_{n,0}+a^\Delta_{n,\ell}\leq
C\sigma_-^{-c_{\rm env}}$ for fixed $c_{\rm env}$.
\end{assumption}

\begin{theorem}[Hierarchical composition at both stages on $(x,t)$]
\label{thm:weak-hierarchy}
Suppose Assumption~\ref{as:weak-hierarchy}, $X_0\in[-1,1]^d$, the uniform
overlap condition, and the independent-training-sample condition hold. Let
$0<\Tmin<\Tmax$, $\Tmin\geq(\log n)^{-c_T}$, $\Tmax\asymp\log n$, and
$n_\ell\gtrsim\log n$.  Fit both stages with coordinatewise
bounded-coefficient ReLU sieves from Appendix~\ref{app:fitted-sieves},
using the respective stage structures and budgets. Apply $\Pi$ to the
pooled sieve and clip the combined score as in Theorem~\ref{thm:oracle}.
For any fixed $b>0$, with probability at least $1-n^{-b}$,
\begin{equation}
\calR^{\rm int}_\ell(\widehat s_\ell)
\leq C\polylog(n)\left\{
 \rho_n(\mathbf P_0^{\rm st})
 +\rho_{n_\ell}(\mathbf P_{\Delta,\ell}^{\rm st})
\right\}.
\label{eq:weak-hierarchy-rate}
\end{equation}
Before absorbing the envelopes into $\polylog(n)$, the two terms on the
right are multiplied by $1+a_{n,0}^2$ and
$1+(a^\Delta_{n,\ell})^2$, respectively. The inverse-polylogarithmic
cutoff makes these factors polylogarithmic. The constant may depend on $b$.
\end{theorem}

\begin{proof}[Proof of Theorem~\ref{thm:weak-hierarchy}]
For $c\in\{0,(\Delta,\ell)\}$, put $m_0=n$,
$m_{\Delta,\ell}=n_\ell$, $a_{n,(\Delta,\ell)}=a^\Delta_{n,\ell}$,
and $T_0=s$, $T_{\Delta,\ell}=\Delta_\ell$.
The constituent approximation and entropy construction in the proof of
Proposition~\ref{prop:stable-offset} applies to either normalized field:
on $Q_n$ it uses only the composition structure. With sample budget $m_c$,
it gives
\[
 \sup_{Q_n}\|f_c^\circ-T_c\|_2^2
 \leq C a_{n,c}^2\polylog(n)\rho_{m_c}(\mathbf P_c^{\rm st}),\qquad
 V_c\leq C m_c\polylog(n)\rho_{m_c}(\mathbf P_c^{\rm st}),
\]
with $\log A_c\leq\polylog(n)$. The clamped comparator is bounded by
$Ca_{n,c}\max\{1,R_n,\Tmax\}$. Since both envelopes and
$\sigma_-^{-1}$ are polylogarithmic, \eqref{eq:sieve-spatial-tail}
makes the complement negligible. Time integration contributes at most
$\Tmax-\Tmin=O(\log n)$, and \eqref{eq:app-projection-risk}
controls score clipping.

Apply Lemma~\ref{lem:app-generic-oracle} to the pooled class and transfer
its risk by overlap as in Proposition~\ref{prop:pooled-general}.
Conditional on that fit, apply the same oracle and
Lemma~\ref{lem:app-clipped-comparator} to the offset class. Their proofs
use only the logarithmic integral bounds. Here
$\Lambda_2\lesssim1+\log_+(1/\Tmin)+\Tmax=O(\log n)$,
and likewise $\Lambda_1=O(\log n)$.
The two contributions are
$C\polylog(n)(1+a_{n,c}^2)\rho_{m_c}(\mathbf P_c^{\rm st})$.
Use confidence $n^{-b}/2$ at each stage and the probability argument
of Theorem~\ref{thm:mixed-main}. Absorbing the envelopes proves
\eqref{eq:weak-hierarchy-rate}.
\end{proof}

\subsection{Besov regularity at both stages}
\label{sec:both-besov}
We write $\rho_m^{\rm B}(s,D)=m^{-2s/(2s+D)}$.
Suppose $p_0$ and $p_{0,\ell}$ satisfy Lemma~\ref{lem:proof-oko-input}
with smoothness $s_0,s_\ell$, respectively. Assume uniform overlap, independent
training samples, \eqref{eq:window}, and $n_\ell\geq n^\epsilon$
for fixed $\epsilon>0$.
We use the pooled Oko class at Stage~1 and the difference of pooled and
group Oko classes at Stage~2. For any fixed $b>0$, the oracle gives, with probability at
least $1-n^{-b}$ and a constant allowed to depend on $b$,
\begin{equation}
 \calR^{\rm int}_\ell(\widehat s_\ell)
 \leq C\polylog(n)\{
 \rho_n^{\rm B}(s_0,d)+\rho_{n_\ell}^{\rm B}(s_0,d)
                      +\rho_{n_\ell}^{\rm B}(s_\ell,d)\}.
\label{eq:both-besov}
\end{equation}
Proposition~\ref{prop:pooled-general} supplies the first-stage term.
For the difference class, Lemma~\ref{lem:proof-oko-input}, overlap, and
Lemma~\ref{lem:app-envelope} give approximation error
$C\polylog(n)(N_0^{-2s_0/d}+N_\ell^{-2s_\ell/d})$.
For each score factor, \eqref{eq:sieve-spatial-tail} and
\eqref{eq:app-projection-risk} control input clamping and score clipping;
overlap transfers the pooled factor's bounds to the target law.
Covering each factor to radius $\eta/2$
adds their entropies, giving cost $(N_0+N_\ell)/n_\ell$.
Choosing $N_c\asymp n_\ell^{d/(2s_c+d)}$ balances both terms.
The polynomial lower bound on $n_\ell$ makes all resolution logarithms
comparable to $\log n$. Apply Theorem~\ref{thm:oracle} to obtain
\eqref{eq:both-besov}. The target stage pays for both full-score
approximations; the simpler-offset gain requires additional structure.
\end{document}